\documentclass[]{article}

\usepackage[utf8]{inputenc}
\usepackage[T1]{fontenc}
\usepackage[hyphens]{url}
\usepackage{authblk}
\makeatletter
\renewcommand\AB@affilsepx{\quad\protect\Affilfont}
\makeatother
\renewcommand\Affilfont{\fontsize{10}{13}\selectfont}
\usepackage{booktabs}
\usepackage{amsfonts}
\usepackage{pifont}
\usepackage{nicefrac}
\usepackage[nopatch=footnote]{microtype}
\usepackage{amsmath, amsthm, amssymb}
\usepackage[
    colorlinks=true,
    linkcolor={red},
    citecolor={blue!50!black},
    urlcolor={blue!70!black}
]{hyperref}
\usepackage[most]{tcolorbox}
\usepackage{fancybox}
\usepackage{graphicx}
\usepackage{listings} 
\usepackage{float}
\usepackage{algorithm}
\usepackage[noend]{algpseudocode}
\usepackage{color}
\usepackage{tikz}
\usepackage{pgfplots}
\pgfplotsset{compat=1.18}
\usepackage{array}
\usepackage{caption}
\usepackage{subcaption}
\usepackage[textsize=scriptsize]{todonotes}
\usepackage{enumitem}
\usepackage{wrapfig}
\usepackage{etoolbox}
\usepackage{diagbox}
\usepackage{comment}
\usepackage{makecell}
\usepackage{multirow}
\usepackage{bbm, bm}
\usepackage{multicol}
\usepackage{mathpazo}
\usepackage{array, booktabs}
\usepackage{arydshln}
\usepackage{siunitx}
\usepackage{times}
\usepackage{xfrac}
\usepackage{setspace}
\usepackage[normalem]{ulem}
\usepackage[round]{natbib}
\usepackage{fontawesome5}
\usepackage{cleveref}
\usepackage{aliascnt}
\usepackage{xcolor}
\newaliascnt{appendixsection}{section}
\crefname{appendixsection}{appendix}{appendices}
\Crefname{appendixsection}{Appendix}{Appendices}
\usepackage[
  textwidth=6.0in,
  textheight=9.0in,
  left=1.25in,
  top=1in
]{geometry}

\setlist[itemize]{leftmargin=2.5em}
\setlist[itemize]{leftmargin=2.5em}
\setlist[enumerate]{leftmargin=2.5em}
\usepackage[framemethod=TikZ]{mdframed}
\mdfsetup{skipabove=5pt,
innertopmargin=0pt}

\tcbuselibrary{breakable}
\lstdefinestyle{modeloutput}{
  basicstyle=\ttfamily\tiny,
  breaklines=true,
  breakatwhitespace=false,
  breakindent=0pt,
  columns=fullflexible,
  keepspaces=true,
  literate={≈}{{$\approx$}}1
}

\newtcolorbox{responsebox}[1]{
  breakable,
  colback=gray!3,
  colframe=gray!40,
  title={#1},
  fonttitle=\bfseries,
  boxrule=0.4pt,
  arc=1mm,
  left=0.5mm,
  right=0.5mm,
  top=0.5mm,
  bottom=0.5mm,
  before skip=0.3em,
  after skip=0.3em
}

\newtcolorbox{promptbox}[1][]{
  colback=gray!5,
  colframe=gray!55,
  boxrule=0.5pt,
  arc=2mm,
  left=6pt,
  right=6pt,
  top=6pt,
  bottom=6pt,
  fonttitle=\bfseries,
  title=#1
}

\newtheoremstyle{slplain}
  {.4\baselineskip\@plus.1\baselineskip\@minus.1\baselineskip}
  {.3\baselineskip\@plus.1\baselineskip\@minus.1\baselineskip}
  {\itshape}
  {}
  {\bfseries}
  {.}
  { }
  {}
\theoremstyle{slplain} 

\newtheorem*{definition*}{Definition}
\newtheorem*{theorem*}{Theorem}

\newtheorem{theorem}{Theorem}[section]
\newaliascnt{lemma}{theorem}
\newtheorem{lemma}[lemma]{Lemma}
\aliascntresetthe{lemma}

\newaliascnt{proposition}{theorem}

\aliascntresetthe{proposition}

\newaliascnt{corollary}{theorem}
\newtheorem{corollary}[corollary]{Corollary}
\aliascntresetthe{corollary}

\newaliascnt{observation}{theorem}

\aliascntresetthe{observation}

\newaliascnt{claim}{theorem}

\aliascntresetthe{claim}

\newaliascnt{definition}{theorem}

\aliascntresetthe{definition}

\newaliascnt{fact}{theorem}

\aliascntresetthe{fact}

\crefname{lemma}{lemma}{lemmas}
\Crefname{lemma}{Lemma}{Lemmas}
\crefname{proposition}{proposition}{propositions}
\Crefname{proposition}{Proposition}{Propositions}
\crefname{corollary}{corollary}{corollaries}
\Crefname{corollary}{Corollary}{Corollaries}

\crefalias{lemma}{lemma}
\crefalias{proposition}{proposition}
\crefalias{corollary}{corollary}

\makeatletter
\newtheorem*{rep@theorem}{\rep@title}
\newcommand{\newreptheorem}[2]{%
\newenvironment{rep#1}[1]{%
 \def\rep@title{#2 \ref{##1}}%
 \begin{rep@theorem}}%
 {\end{rep@theorem}}}
\makeatother

\newreptheorem{theorem}{Theorem}
\newreptheorem{lemma}{Lemma}
\newreptheorem{claim}{Claim}
\newreptheorem{corollary}{Corollary}
\newreptheorem{proposition}{Proposition}

\theoremstyle{definition}

\theoremstyle{plain} 

\numberwithin{equation}{section}

\DeclareMathOperator{\E}{\mathbb{E}}

\DeclareMathOperator{\KL}{\mathrm{KL}}

\renewcommand\bar\overline

\newcolumntype{C}[1]{>{\centering\let\newline\\\arraybackslash\hspace{0pt}}m{#1}}

\newcommand{\calB}{\ensuremath{\mathcal{B}}}

\newcommand{\calL}{\ensuremath{\mathcal{L}}}

\newcommand{\calR}{\ensuremath{\mathcal{R}}}

\newcommand{\calV}{\ensuremath{\mathcal{V}}}

\def\nd/{\textsuperscript{nd}}
\def\rd/{\textsuperscript{rd}}
\def\th/{\textsuperscript{th}}

\makeatletter
\def\nnil{\nil}
\newcounter{prob}

\newcounter{dual}

\makeatother

\newenvironment{prob*}{%
	\csname equation*\endcsname%
	\aligned%
}{%
	\endaligned%
	\csname endequation*\endcsname%
}

\title{When to Follow the Teacher:\\Advantage-Based On-Policy Distillation}
\author[1]{Mahdi Sabbaghi\thanks{Correspondence can be made to: smahdi@seas.upenn.edu}}
\author[1]{George Pappas}
\author[1]{Hamed Hassani}
\author[2]{Adel Javanmard}

\affil[1]{University of Pennsylvania}
\affil[2]{University of Southern California}
\date{}

\begin{document}
\maketitle

\begin{abstract}
Online Policy Distillation (OPD) often outperforms offline methods such as Supervised Fine-Tuning (SFT) by training on trajectories sampled directly from the student. However, standard OPD trains the student to follow the teacher at every state of the trajectory, assuming that teacher guidance is always beneficial. We instead distinguish between what the teacher prefers and what actually helps the current student succeed. We introduce Advantage-Based On-Policy Distillation (AOPD), which weights teacher actions by their counterfactual advantage under the student. Teacher actions that improve the student’s expected outcome are reinforced, while those that do not are downweighted.

Directly computing these advantages requires trying alternative teacher actions at each step and running new rollouts from each. Instead, we show that the AOPD objective can be estimated using only standard student rollouts and their terminal rewards. Our estimator uses teacher importance weights on the student-sampled actions to provide unbiased advantage-weighted supervision, without any additional rollouts or reward evaluations. This allows AOPD to focus teacher supervision on actions that help the current student while preserving the rollout efficiency of standard OPD. Experiments on math, coding, and Medqa show gains over distillation and RL baselines in both external-teacher and self-distillation settings. 

\begin{center}
    \faGithub~\url{https://github.com/Helloworld10011/AOPD}
\end{center}
\end{abstract}

\section{Introduction}
Knowledge distillation transfers capabilities from a stronger teacher to a smaller student model by training the student to imitate the teacher's outputs.
Offline distillation is a common approach that uses fixed teacher-generated trajectories. But this approach can lead to a mismatch between the training trajectories and the states the student generates itself. 
On-policy distillation (OPD) addresses this mismatch by querying the teacher along student-sampled trajectories, providing dense token-level supervision at the states the student actually visits \citep{agarwal2024policy,lu2025onpolicydistillation,song2026survey,ross2011reduction,li2026rethinking}.
Although the trajectories now come from the student, the objective remains to match the teacher, without distinguishing which teacher actions improve the student's chance of success \citep{liu2026teacher}.

This limitation also applies to on-policy self-distillation (OPSD), which replaces the external teacher with a privileged version of the student itself, such as conditioned on ground-truth solutions \citep{zhao2026self,hubotter2026reinforcement}.
Whether external or privileged, the teacher may choose an action from which it can solve the problem, while the student cannot complete the remaining steps correctly.
Conversely, distillation discourages a valid student solution when it differs from the teacher's preferred approach.
Thus the question is whether the teacher action at every state increases the student's chance of success.
An action should be credited according to its impact on the student's task success, not simply the teacher's preference \citep{chen2026look,saadi2026validity}.

Prior attempts to mitigate this optimization misalignment rely on making the teacher supervision more selective. A common approach combines outcome-based reinforcement learning with distillation, but adding a reward term does not change which teacher actions the distillation term reinforces \citep{xu2025kdrl,zhang2026reinforcement}.
Other methods downweight distillation when the teacher is uncertain, using entropy as a proxy, or when teacher tokens are implausible under the student \citep{jin2026entropy,wang2026not}.
Alternative approaches validate teacher actions using auxiliary reward models \citep{saadi2026validity,gu2026rethinking} or explicit intervention rollouts \citep{chen2026look}.
These methods provide ways to identify useful teacher guidance, but rely on auxiliary signals or additional counterfactual trajectories. More importantly, they do not formulate the objective directly around the impact of the teacher action on the student's expected task reward. Therefore, we ask:
\begin{quote}
    \textit{How should a distillation objective weight teacher supervision to improve the student’s eventual task success?}
\end{quote}

\noindent We propose \textbf{Advantage-based On-Policy Distillation} (AOPD), the first formulation of on-policy distillation that directly weights teacher actions by their advantage under the current student.
Intuitively, the advantage measures whether taking a teacher action and then letting the student finish the response improves expected final reward over following the student's own policy. Since this advantage is defined with respect to the current student, the same teacher action can be useful for one student but unhelpful for another.
AOPD therefore reinforces teacher actions with positive advantage and suppresses those with negative advantage. This is illustrated in \Cref{fig:intro} (left).

A direct estimate of these advantages would require generating new student continuations after alternative teacher actions.
We instead derive an unbiased importance-sampling estimator of the AOPD objective using only the student's on-policy trajectories and their terminal rewards.
This allows AOPD to incorporate outcome-level rewards into token-level teacher supervision without learning a value function, generating counterfactual continuations or teacher rollouts, or relying on an external process reward model. 

\noindent\textbf{Our contributions} are summarized as follows:
\begin{itemize}[leftmargin=0.8em]
    \item \textbf{AOPD algorithm.}
    In \Cref{sec:method}, we provide an advantage-weighted formulation of on-policy distillation and derive an unbiased estimator of its objective, without additional rollouts or auxiliary models. We also discuss clipping the importance sampling weights.

    \item \textbf{Experimental results.}
    In \Cref{sec:experiments}, we evaluate multiple model sizes on tasks such as math and coding under both external-teacher and self-distillation settings, showing improvements over OPD, OPSD, and GRPO. We further examine whether teacher actions are uniformly beneficial in an ablation study.  

    \item \textbf{Reward improvement beyond imitation.}
    In \Cref{sec:atd-theory}, we establish a reward-improvement guarantee for AOPD under an exact KL-proximal analysis and show that forward-KL OPD can reduce reward even with a stronger teacher. We further construct a two-state counterexample showing that simply adding an RL objective to OPD does not guarantee reward improvement.
\end{itemize}
\begin{figure}
    \centering
    \includegraphics[width=1.0\linewidth]{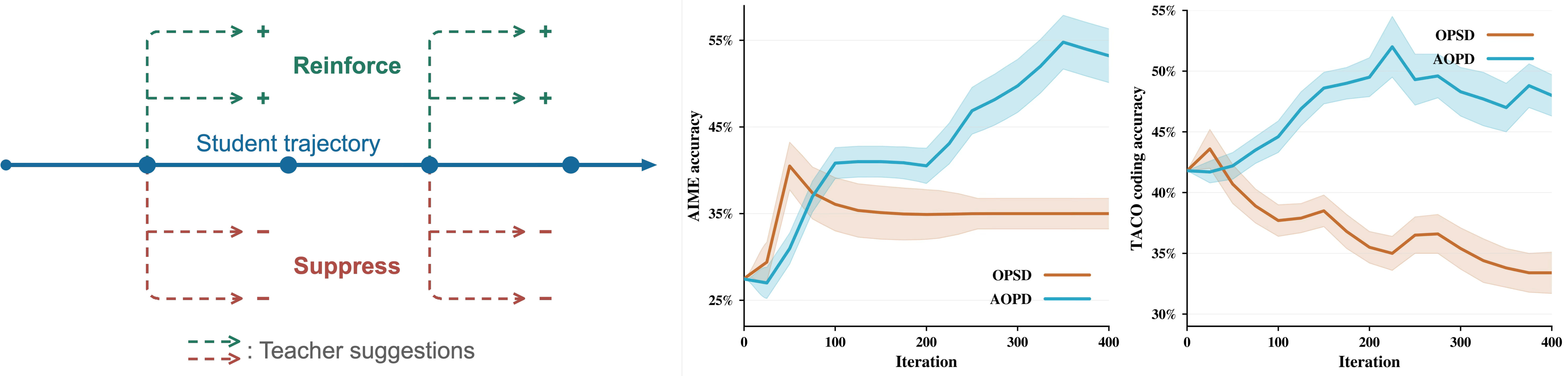}
    \caption{\textbf{(left)} The overall picture of our method. On the student sampled trajectories, not all the suggested actions by the teacher are useful to the student, and many suggestions are destructive. \textbf{(middle)} Performance comparison of our method with OPSD when training Qwen3-8B in non-thinking mode on Open-Thoughts. \textbf{(right)} Comparison with OPSD when training Qwen3-4B-Instruct on TACO for coding.}
    \label{fig:intro}
    \vspace{-8pt}
\end{figure}

\section{Related Work}
\paragraph{Offline supervision.}
Supervised fine-tuning and offline knowledge distillation learn from fixed demonstrations or teacher-generated trajectories \citep{hinton2015distilling,kim2016sequence,wen2023f,gu2024minillm}. As established in imitation learning, offline supervision suffers from covariate shift and compounding errors when the learned policy visits states absent from the training distribution \citep{ross2011reduction,bengio2015scheduled,agarwal2024policy}. This mismatch is particularly relevant for autoregressive reasoning, where early deviations change all subsequent steps drastically \citep{gu2024minillm,agarwal2024policy}.

\vspace{-7pt}
\paragraph{On-policy reinforcement learning.}
On-policy RL instead learns from trajectories generated by the current policy and has become central to LLM alignment and reasoning through methods such as PPO, RLOO, GRPO, and related policy-gradient algorithms \citep{schulman2017proximal,ouyang2022training,ahmadian2024back,shao2024deepseekmath,sabbaghi2026robust}. RLVR  directly optimizes task reward but typically provides sparse outcome-level supervision and can struggle when useful trajectories are rarely sampled \citep{casper2023open,yue2025does,xu2025kdrl,xiong2025minimalist}.

\vspace{-7pt}
\paragraph{On-policy distillation.}
On-policy distillation combines on-policy sampling with dense teacher supervision by matching teacher and student distributions along student-generated trajectories \citep{agarwal2024policy,gu2024minillm,lu2025onpolicydistillation,li2026rethinking,song2026survey}. OPSD further uses a privileged version of the student itself as the teacher, often conditioning it on ground-truth solutions or other privileged information \citep{zhao2026self,hubotter2026reinforcement,penaloza2026privileged,ye2026policy}. These approaches correct \emph{where} supervision is provided, but generally still apply teacher supervision uniformly and throughout the student's trajectory.

\vspace{-7pt}
\paragraph{Selective teacher supervision.}
Recent methods make this supervision selective using teacher uncertainty, teacher--student compatibility, local validity, verifier signals, or additional intervention rollouts \citep{jin2026entropy,wang2026not,saadi2026validity,gu2026rethinking,chen2026look,zhang2026incentivizing}. These criteria identify when teacher guidance may be useful, but do not directly define usefulness by its effect on the student's eventual task reward. AOPD instead introduces a new formulation for on-policy distillation that weights each teacher action by its advantage under the current student and estimates this signal from the existing on-policy rollouts and terminal rewards.

\vspace{-5pt}
\section{Method: Advantage-based On-Policy Distillation}\label{sec:method}
\vspace{-4pt}

Let $S$ denote the frozen student policy used to generate the current rollout batch, $S_\theta$ the policy being updated, and $T$ the teacher policy. Given a prompt $x$, the student generates a response $y=(y_0,\ldots,y_{H-1})$ according to
\(
S(y\mid x)=\prod_{t=0}^{H-1}S(y_t\mid h_t),
 h_t=(x,y_{<t})
\), then $h_t\sim S(\cdot\mid x)$ denotes a sampled trajectory up to step $t$. 
The terminal reward $R(x,y)\in [0,1]$ indicates task success. We assume a shared token vocabulary $\calV$ and access to teacher probabilities $T(\cdot\mid h_t)$ at student-generated histories.

\vspace{-5pt}
\paragraph{Reverse-KL distillation.}
Reverse-KL distillation \citep{agarwal2024policy} minimizes:
\begin{equation}
\label{eq:rkl}
\calL_{\mathrm{RKL}}(\theta)
=
\frac{1}{H}\sum_{t=0}^{H-1}
\E_{h_t \sim S(.)}
\left[
\KL\bigl(S_\theta(\cdot\mid h_t)\|T(\cdot\mid h_t)\bigr)
\right].
\end{equation}
At each step, the actions come from $a\sim S(\cdot\mid h_t)$ and the expectation is over log-ratio
$\log\frac{T(a\mid h_t)}{S(a\mid h_t)}$. Thus, the teacher evaluates actions already sampled by the student.

\vspace{-5pt}
\paragraph{Forward-KL distillation.}
Forward-KL distillation instead samples the action from teacher's distribution at each step of the student sampled trajectory and minimizes the log-ratio $\log\frac{S(a\mid h_t)}{T(a\mid h_t)}$ :
\begin{equation}
\label{eq:fkl}
\calL_{\mathrm{FKL}}(\theta)
=
\frac{1}{H}\sum_{t=0}^{H-1}
\E_{h_t \sim S(.)}
\left[
\KL\bigl(T(\cdot\mid h_t)\|S_\theta(\cdot\mid h_t)\bigr)
\right].
\end{equation}
Dropping the teacher entropy that is independent of $\theta$, \Cref{eq:fkl} is equivalent to the following maximization objective:

\begin{equation}
\label{eq:fkl_max}
J_{\mathrm{FKL}}(\theta)
=
\E_{x,\;y\sim S(\cdot\mid x)}
\left[
\frac{1}{H}\sum_{t=0}^{H-1}
\E_{a\sim T(\cdot\mid h_t)}
\left[\log S_\theta(a\mid h_t)\right]
\right].
\end{equation}
The objective reinforces teacher-supported actions at student-generated histories $h_t$. In practice, FKL and RKL objectives roughly reach the same performance, and we will develop our method for FKL. Critically, FKL reinforces the teacher actions without distinguishing whether they improve the student's final task outcome. 

\vspace{-5pt}
\paragraph{On-policy self-distillation.}
On-policy self-distillation (OPSD) \citep{zhao2026self} uses the same objective above, but it replaces the external teacher with a privileged version of the student itself. Given privileged information $y^*$, such as the ground-truth solution, the teacher distribution becomes:
\[
T_{\mathrm{OPSD}}(a\mid h_t)
=
S(a\mid h_t,y^*)
\]
and the rollout is still generated from the unprivileged student $S(\cdot\mid h_t)$. Our method makes no assumptions about how $T$ is obtained and therefore applies directly to both OPD and OPSD.

\vspace{-5pt}
\paragraph{Our method: Advantage-based on-policy distillation (AOPD).}
Starting from the FKL objective in \Cref{eq:fkl_max}, we preserve the teacher supervision at the same student-generated histories but change how teacher actions are weighted. Rather than reinforcing every teacher-supported action, we weight each action by its effect on the student's expected reward. For a history $h=(x,y_{<t})$, define:
\begin{align}\label{eq:student_values}
Q^S(h_t,a)
&=
\E_{y_{>t}\sim S(\cdot\mid h_t a)}
\left[R(x,y_{<t}\;a\; y_{>t})\right], \quad
V^S(h_t)
=
\E_{a\sim S(\cdot\mid h_t)}[Q^S(h_t,a)]
\nonumber\\
A^S(h_t,a)
&=
Q^S(h_t,a)-V^S(h_t).
\end{align}
These are the standard RL value and advantage function. The advantage $A^S$ compares the expected reward after taking $a$ with the expected reward when the student follows its own action policy. We therefore modify the FKL objective by weighting each teacher action by $A^S(h_t,a)$:
\begin{equation}
\label{eq:step_AOPD}
J_{\mathrm{AOPD}}(\theta)
=
\E_{x,\;y\sim S(\cdot\mid x)}
\left[
\frac{1}{H}\sum_{t=0}^{H-1}
\E_{a\sim T(\cdot\mid h_t)}
\left[
A^S(h_t,a)\log S_\theta(a\mid h_t)
\right]
\right].
\end{equation}
Unlike the forward KL objective, AOPD assigns positive weights to teacher actions that improve the student's expected reward and negative weights to those that reduce it. Both the advantage and the history distribution are defined under the frozen rollout policy $S$. In \Cref{thm:atd-improvement}, we show that our objective increases the average reward with guarantee, whereas OPD has an arbitrary sign that depends on the advantage of teacher's actions. 

\vspace{-5pt}
\paragraph{Estimating AOPD from student rollouts.}
Directly evaluating $Q^S$ for teacher actions would require separate student continuations from the teacher interventions. Instead, we estimate the entire objective using a student rollout $y^S\sim S(\cdot\mid x)$ and its terminal reward $R^S=R(x,y^S)$. At each step, let $a_t=y_t^S$ and define
\(
w_t=\frac{T(a_t\mid h_t)}{S(a_t\mid h_t)}
\).
The ratio $w_t$ converts expectations under the student into expectations under the teacher.
We use centered rewards $\widetilde R^S=R^S-b$ for a group with $G$ responses, where $b$ is a baseline independent of $y^S$ (e.g., the leave-one-out mean). The resulting estimator is:
\begin{equation}
\label{eq:AOPD_estimator}
\boxed{
\widehat J_{\mathrm{AOPD}}(\theta)
=
\frac{\widetilde R^S}{H}
\sum_{t=0}^{H-1}
\Big[
w_t\log S_\theta(a_t\mid h_t)- \E_{a\sim T} \big[\log S_\theta(a\mid h_t)
\big] \Big], \quad w_t := \frac{T(a_t\mid h_t)}{S(a_t\mid h_t)}
}
\end{equation}
In the following lemma, we show that the two terms within the sum estimate the contributions of $Q^S-b$ and $V^S-b$, respectively. So their difference obtains the desired advantage in \Cref{eq:step_AOPD}.

\begin{lemma}[Unbiased AOPD estimator]
\label{lem:AOPD_unbiased}
Assume that $T(a\mid h)>0$ implies $S(a\mid h)>0$ at every visited step. We have:
\[
\E\left[\widehat J_{\mathrm{AOPD}}(\theta)\right]
=
J_{\mathrm{AOPD}}(\theta)
\]
And the same identity holds for their gradients.
\end{lemma}

The proof is given in Appendix~\ref{app:proofs}. 
When $\widetilde R^S>0$, AOPD reinforces the successful student action, with a larger update when the teacher also assigns it high probability. Unlike GRPO, this focuses the positive update on successful actions supported by the teacher. The second term suppresses alternative teacher actions, since the sampled student action already led to success. When $\widetilde R^S<0$, the signs reverse: AOPD suppresses the failed student action, especially when the teacher also favors it, while promoting alternative teacher actions. Thus, AOPD uses the teacher to distinguish which actions should be reinforced or corrected based on the rollout outcome.


\vspace{-5pt}
\paragraph{Clipping the importance sampling.}
Large importance sampling weights can make rare sampled tokens dominate the update, so we cap the importance sampling as:
\begin{equation}\label{eq:clipping}
    \widetilde w_t=\min\left(\frac{T(a_t\mid h_t)}{S(a_t\mid h_t)},\,c\right)
\end{equation}
In our experiments this constant is $c = 4$. We apply the same truncation to teacher probabilities in the second term without re-normalization:
$\widetilde T(a | h)=\min\{T(a | h),cS(a | h)\}$. The resulting estimator is:
\begin{equation}
\label{eq:AOPD_clipped}
\widetilde J_{\mathrm{AOPD}}(\theta)
=
\frac{\widetilde R^S}{H}
\sum_{t=0}^{H-1}
\left[
\widetilde w_t\log S_\theta(a_t\mid h_t)
-
\sum_{a\in\calV}
\widetilde T(a\mid h_t)\log S_\theta(a\mid h_t)
\right].
\end{equation}
We show in \Cref{cor:atd-clipping} that applying the same truncation to both terms preserves the advantage-weighted objective and its reward-improvement guarantee. We further discuss the impact of clipping on training in Appendix \ref{app:additional_exp}. \uline{\mbox{\Cref{alg:AOPD}} summarizes the final version of our algorithm.}

\begin{algorithm}[t]
\caption{AOPD with the unbiased estimator}
\begin{algorithmic}[1]
\Require Student $S_\theta$, teacher $T$, reward function $R$, group size $G$
\For{each training iteration}
    \State Freeze the rollout policy $S\leftarrow S_\theta$.
    \For{each prompt $x\in\calB$}
        \State Sample $G$ independent responses $y^{(i)}\sim S(\cdot\mid x)$.
        \State Compute terminal rewards $R_i=R(x,y^{(i)})$.
        \State Set $b_i=\frac{1}{G-1}\sum_{j\ne i}R_j$ for $G>1$.
        \State Evaluate student and teacher probabilities at every visited history.
        \State Compute $\widetilde J_{\mathrm{AOPD}}^{(i)}(\theta)$ using \Cref{eq:AOPD_clipped}.
    \EndFor
    \State $\calL(\theta)\leftarrow-\frac{1}{|\calB|G}\sum_{x\in\calB}\sum_{i=1}^G\widetilde J_{\mathrm{AOPD}}^{(i)}(\theta)$  \Comment{differentiating only through $S_\theta$.}
\EndFor
\end{algorithmic}
\label{alg:AOPD}
\end{algorithm}

\section{Experiments}\label{sec:experiments}
We evaluate AOPD in two settings: \textbf{(i)} distillation from a stronger external teacher and \textbf{(ii)} self-distillation from a version of the student with privileged information. We compare AOPD against existing distillation and RL methods across math, coding, biology, and medical question answering.

\subsection{Experimental Setup}

\vspace{-5pt}
\paragraph{Tasks.} We conduct our experiments across four tasks:\\
\textbf{Math:} We train on Open-Thoughts \citep{guha2025openthoughtsdatarecipesreasoning} and evaluate on held-out Open-Thoughts problems, AIME24, AIMO-Validation-AMC (AMC), and HMMT \citep{aimovalidationaime,aimovalidationamc,dekoninck2026matharena}.
\par\noindent
\textbf{Coding:} We train on verified TACO problems \citep{li2023taco,likaixin2024taco-verified} and evaluate on HumanEval, HumanEval+, and MultiPL-E \citep{chen2021evaluating,cassano2022multipl}.
\par\noindent
\textbf{Biology:} We train on SciKnowEval \citep{feng2024sciknoweval} and evaluate on 50 held-out biology questions and the biology subset of MMLU-Pro \citep{wang2024mmlu}.
\par\noindent
\textbf{Medical QA:} We train on MedMCQA \citep{pal2022medmcqa} and evaluate on held-out MedMCQA and the separate MedQA benchmark \citep{jin2021disease}.

\vspace{-5pt}
\paragraph{Models and configurations.}
Following previous work \citep{zhao2026self,ye2026policy}, we use Qwen2.5 and Qwen3 models of different sizes \citep{yang2025qwen3}.
For external-teacher distillation, we use Qwen3-1.7B as the student, with Qwen3-4B-Instruct as the teacher for math training and Qwen3-8B as the teacher for coding.
For self-distillation, we use Qwen3-1.7B and Qwen3-8B for math, and Qwen2.5-7B-Instruct and Qwen3-4B-Instruct for biology, medical question answering, and coding.
\uline{All Qwen3 experiments use non-thinking mode for both the student and teacher.}

\vspace{-5pt}
\paragraph{Baselines.}
For external-teacher distillation, we compare AOPD with OPD using forward-KL (FKL) and reverse-KL (RKL) objectives \citep{agarwal2024policy,lu2025onpolicydistillation}, GRPO \citep{shao2024deepseekmath}, SFT, and KDRL \citep{xu2025kdrl}, which combines OPD and GRPO.
For self-distillation, we compare with OPSD and GRPO, using the forward-KL objective for OPSD following previous work \citep{zhao2026self,jin2026entropy}.
Within each comparison, we use the same number of training iterations, a batch size of 32, and the same LoRA adapters with $(r = 64, \alpha = 128)$, while selecting method-specific hyperparameters.

\vspace{-5pt}
\paragraph{Hyperparameters.}
To preserve the support required by the importance-sampling estimator in \Cref{eq:AOPD_estimator}, we disable top-$p$ and top-$k$ pruning for AOPD.
We use a sampling temperature of $T=0.8$, which we found effective for generating coherent Qwen responses.
For OPD and OPSD, we use $T=1.1$, top-$p=0.95$, and top-$k=20$, following \citet{zhao2026self} and the settings that reproduce their best performance.
For the main Qwen3 math and coding experiments, we use a learning rate of $2\times10^{-5}$ for AOPD.

\subsection{External-Teacher Distillation}\label{sec:teacher_distill}

We first evaluate AOPD with stronger external teachers on math and coding. For math, we use Qwen3-4B-Instruct as the teacher because it performs better on Open-Thoughts than Qwen3-8B in non-thinking mode.
All methods train for 450 iterations on Open-Thoughts and 350 iterations on TACO.
Math training uses a binary correctness reward. For coding, we use partial rewards based on five test cases, which we found more effective for training. We use $G=8$ for GRPO and KDRL, following \citet{xu2025kdrl}, and $G=4$ for OPD and AOPD. This is because our ablations show that larger groups provide little accuracy benefit for OPD and AOPD (Appendix~\ref{app:additional_exp}).
For math, we use a maximum completion length of 8,192 tokens during training and 16,192 tokens at evaluation.

\vspace{-5pt}
\paragraph{Results.}
 \Cref{fig:OPD_results} compares OPD (FKL), GRPO, and AOPD during training.
On math, OPD improves quickly and then plateaus, consistent with observations in previous work \citep{yang2026oprd,tan2026atod}. On the other hand, GRPO improves more slowly and requires more iterations to converge. AOPD also improves rapidly at the beginning but also continues to improve later in training.
This is because unlike standard OPD, which enforces uniform supervision, our objective dynamically adapts which teacher signals to follow as the student improves.
Although OPD and GRPO achieve comparable Open-Thoughts evaluation rewards, OPD performs  better on AIME24, as shown in \Cref{fig:OPD_results} (middle).

On TACO, OPD makes little progress in evaluation reward, whereas AOPD improves and remains above GRPO (\Cref{fig:OPD_results}, right).
Nevertheless, OPD (FKL) improves HumanEval and MultiPL-E performance over the base model (\Cref{table:OPD_main_results}). Overall, our finding across the experiments is that the all distillation methods achieve better performance on benchmarks regardless of the training reward.

\begin{figure*}
    \centering
    \begin{subfigure}{0.333\linewidth}
        \includegraphics[width=\linewidth]{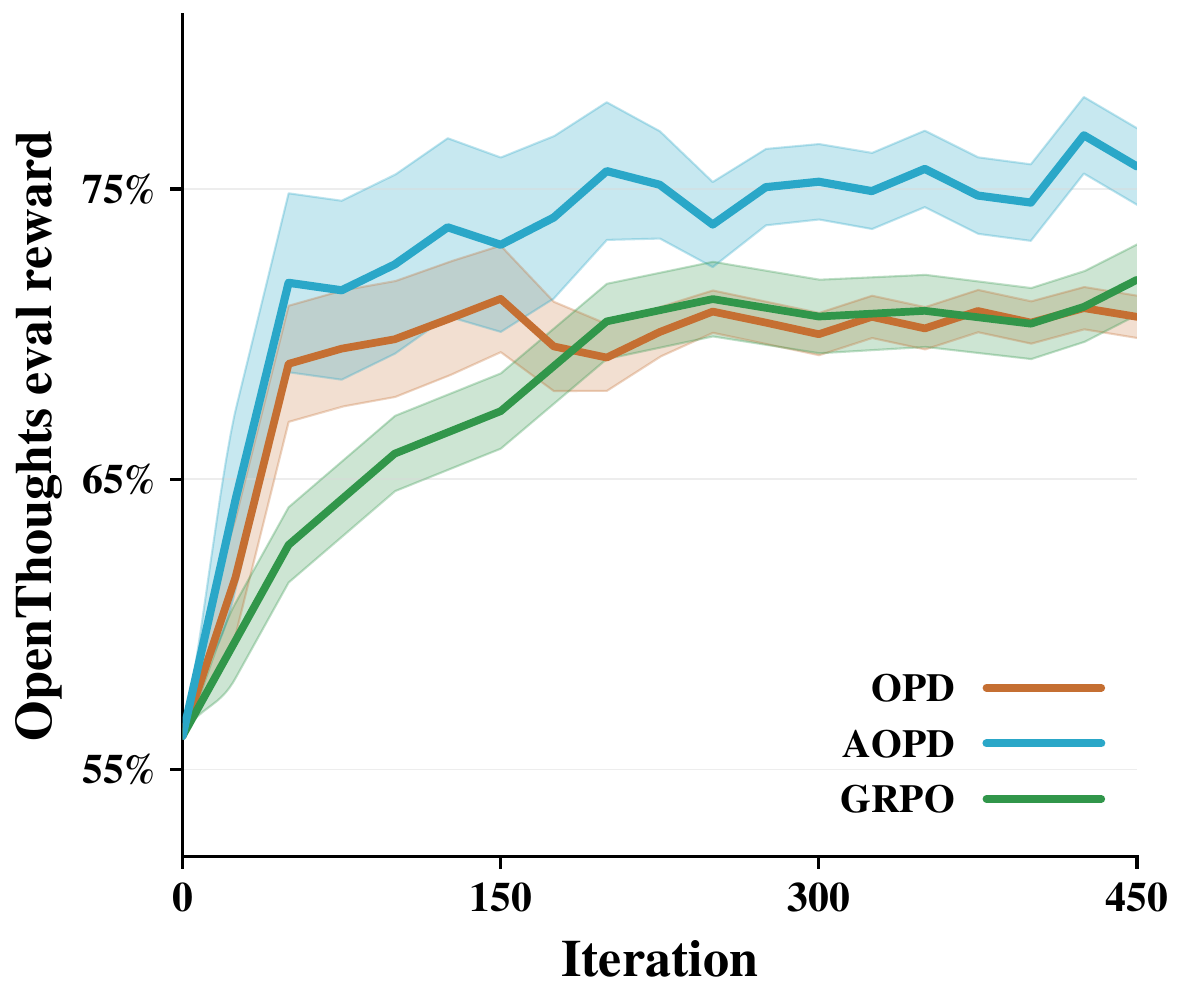}
    \end{subfigure}%
    \begin{subfigure}{0.333\linewidth}
        \includegraphics[width=\linewidth]{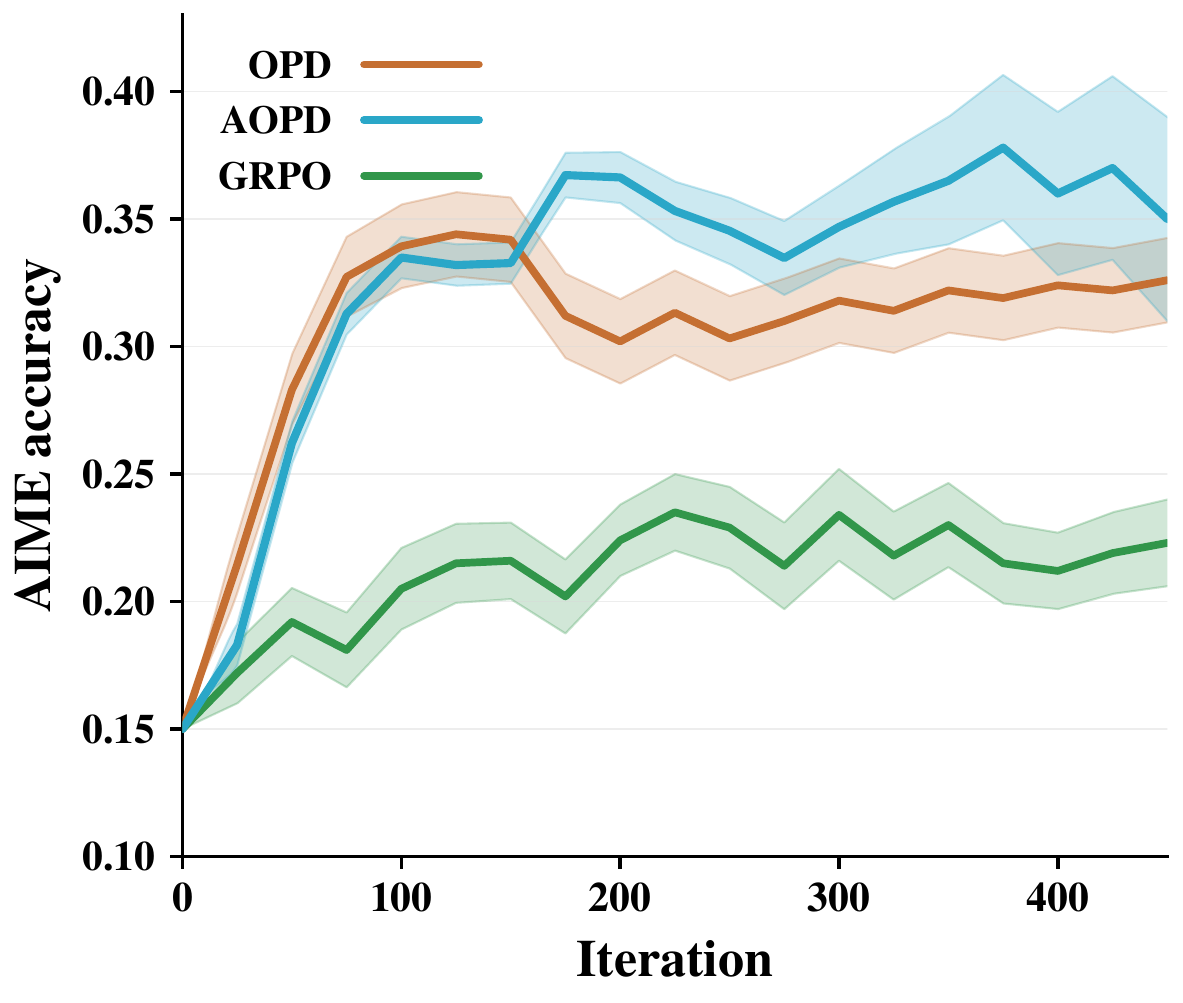}
    \end{subfigure}%
    \begin{subfigure}{0.333\linewidth}
        \includegraphics[width=\linewidth]{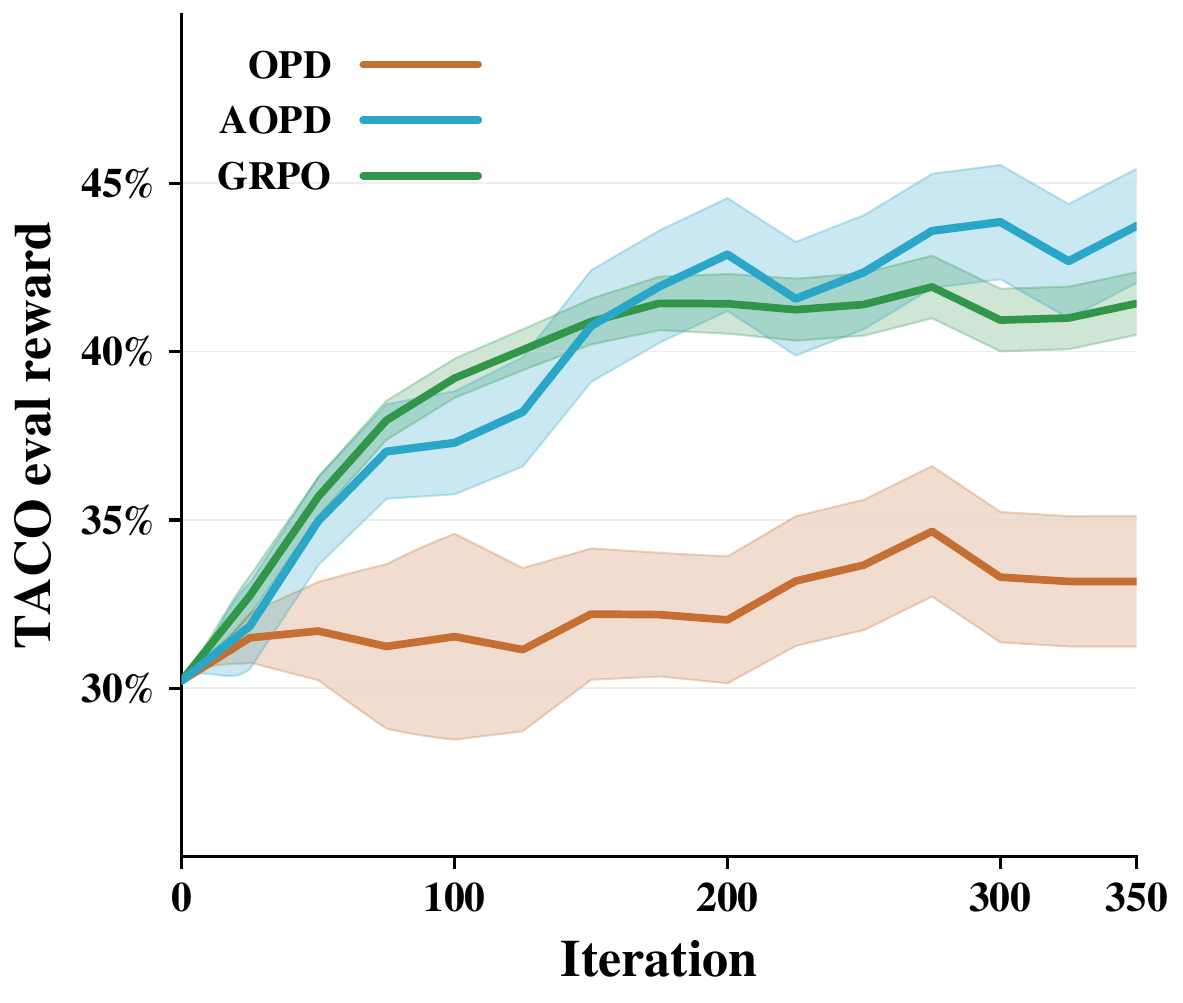}
    \end{subfigure}
    \vspace{-15pt}
    \caption{Training progress under external-teacher distillation with Qwen3-1.7B as the student. \textbf{(Left/middle)} Open-Thoughts evaluation reward and AIME24 performance. \textbf{(Right)} TACO evaluation reward for coding.}
    \label{fig:OPD_results}
\end{figure*}

\Cref{table:OPD_main_results} shows the comparison for all methods across the math and coding benchmarks. For math, AOPD achieves the highest pass@1 on Open-Thoughts, AIME24, and AMC. It improves AIME24 by 2.4 percentage points and AMC by 3.3 points over the baselines.
For coding, AOPD achieves the best scores on HumanEval. However, GRPO achieves the highest HumanEval pass@10, while OPD (FKL) performs the best on MultiPL-E.

\begin{table*}[t!]
  \centering
  \vspace{-4pt}
  \resizebox{\textwidth}{!}{%
    \begin{tabular}{@{}lcccccccc@{}}
      \toprule
      & \multicolumn{4}{c}{\textbf{Math Training on Open-Thoughts}}
      & \multicolumn{4}{c}{\textbf{Code Training on TACO}} \\
      \cmidrule(lr){2-5} \cmidrule(lr){6-9}
      & OpenThoughts & AIME24 & AMC & HMMT
      & \multicolumn{3}{c}{HumanEval}
      & MultiPL-E \\
      Models
      & pass@1 & pass@1 & pass@1 & pass@1
      & HE & HE+ & pass@10 & Avg.\ (7 lang.) \\
      \midrule
      \textbf{Qwen3-1.7B} & 56.6 & 14.7 & 40.1 & 5.8 & 68.3 & 61.6 & 82.5 & 31.4\\
      + SFT & 68.0 & 23.3 & 53.5 & 12.6 & 69.5 & 62.3 & 82.4 & 31.0\\
      + GRPO & 71.8 & 22.3 & 51.4 & 10.6 & 71.8 & 64.0 & \textbf{87.9} & 31.1\\
      + KDRL & 72.8 & 25.0 & 55.8 & 12.1 & 71.5 & 66.7 & 84.4 & 31.0\\
      + OPD (FKL) & 73.2 & 32.1 & 61.8 & \textbf{16.0} & 71.3 & 66.5 & 85.3 & \textbf{33.3}\\
      + OPD (RKL) & 72.0 & 32.6 & 61.0 & 14.2 & 68.3 & 62.2 & 81.8 & 33.2\\
      + AOPD & \textbf{78.4} & \textbf{35.0} & \textbf{65.1} & 15.8 & \textbf{74.4} & \textbf{68.9} & 86.5 & 32.5\\
      \bottomrule
    \end{tabular}
  }
  \vspace{-6pt}
  \caption{External-teacher distillation on math and coding with Qwen3-1.7B as the student. The teachers are Qwen3-4B-Instruct for math and Qwen3-8B for coding.}
  \label{table:OPD_main_results}
\end{table*}

\subsection{Self-Distillation}\label{sec:self_distill}

We next replace the external teacher with a version of the student conditioned on privileged information.
For math, we use Qwen3-1.7B and Qwen3-8B, training for 400 iterations on Open-Thoughts with $G=4$ for OPSD and AOPD. The maximum training completion length is 8,192 tokens.
For the remaining tasks, we use Qwen2.5-7B-Instruct and Qwen3-4B-Instruct for biology, medical question answering, and coding. We train models for 300 iterations with a maximum completion length of 2,048 tokens on medical QA, and for 400 iterations with a maximum completion length of 4,096 tokens on biology and coding. 

\vspace{-5pt}
\paragraph{Math results.}
\Cref{table:OPSD_main_results} compares GRPO, OPSD, and AOPD for both students.
As \Cref{fig:OPSD_results} shows, for Qwen3-1.7B, OPSD makes little progress on Open-Thoughts, consistent with previous work \citep{li2026policy}. This suggests limited benefit from the self-distillation signal in this setting.
Nevertheless, AOPD reaches a comparable evaluation reward to GRPO and outperforms the baselines on all three math benchmarks.
On AIME24, AOPD reaches 35.0\%, improving over GRPO by 12.7 percentage points and OPSD by 19.4 points.
For Qwen3-8B, OPSD is more effective and improves performance across the benchmarks. AOPD again achieves substantial gains in evaluation reward and benchmark performance, as illustrated in \Cref{table:OPSD_main_results} and \Cref{fig:intro} (middle).
In particular, AOPD reaches 52.1\% on AIME24, compared with 41.7\% for OPSD, a gain of 10.4 percentage points.

\vspace{-5pt}
\paragraph{Results with instruct models.}
We further evaluate AOPD across biology, medical question answering, and coding.
As shown in \Cref{table:self_instmodels_working}, for biology, AOPD achieves the highest accuracy on SciKnowEval Biology and MMLU-Pro with Qwen2.5-7B, reaching 54.0\% and 73.8\%, respectively.
For medical question answering, AOPD outperforms GRPO and OPSD on both MedQA and MedMCQA with both models. In particular, with Qwen2.5-7B-Instruct, AOPD reaches 74.8\% on MedQA and 69.8\% on MedMCQA, improving over OPSD by 2.7 and 3.2 percentage points. The training progress on MedQA is also shown in \Cref{fig:OPSD_results}.
For coding, AOPD improves the evaluation reward significantly as shown in \Cref{fig:intro} (right). AOPD achieves the best HumanEval+ scores with both models. On MultiPL-E, AOPD reaches 61.1\% with Qwen3-4B-Instruct, outperforming GRPO by 5.9 points, although it does not outperform GRPO with Qwen2.5-7B-Instruct.

\begin{table*}[t!]
\centering
\vspace{-4pt}
\small
\setlength{\tabcolsep}{5pt}
\begin{tabular}{@{}lcccccccc@{}}
\toprule
& \multicolumn{4}{c}{\textbf{Qwen3-1.7B}}
& \multicolumn{4}{c}{\textbf{Qwen3-8B}} \\
\cmidrule(lr){2-5} \cmidrule(lr){6-9}
Method
& OpenThoughts & AIME24 & AMC & HMMT
& OpenThoughts & AIME24 & AMC & HMMT \\
& pass@1 & pass@1 & pass@1 & pass@1
& pass@1 & pass@1 & pass@1 & pass@1 \\
\midrule
Base
& 56.6 & 14.7 & 40.1 & 5.8
& 68.8 & 27.5 & 59.6 & 14.2 \\
+ GRPO
& \textbf{71.8} & 22.3 & 51.4 & 10.6
& 78.5 & 30.6 & 63.8 & 13.6 \\
+ OPSD
& 55.2 & 15.6 & 45.7 & 6.4
& 72.7 & 41.7 & 75.0 & 18.4 \\
+ AOPD
& 70.2 & \textbf{35.0} & \textbf{61.4} & \textbf{14.2}
& \textbf{84.6} & \textbf{52.1} & \textbf{78.6} & \textbf{20.8} \\
\bottomrule
\end{tabular}
\vspace{-6pt}
\caption{Math results for self-distillation on Open-Thoughts with Qwen3-1.7B and Qwen3-8B, both in non-thinking mode.}
\label{table:OPSD_main_results}
\vspace{-5pt}
\end{table*}

\begin{table*}[t]
\centering
\small
\setlength{\tabcolsep}{6pt}
\begin{tabular}{@{}lllrrrr@{}}
\toprule
& Model & Benchmark & Base & GRPO & OPSD & AOPD \\
\midrule
\multirow{4}{*}{Biology}
& \multirow{2}{*}{Qwen2.5-7B-Instruct} & SciKnowEval & 20.0 & 42.0 & 41.0 & \textbf{54.0} \\
& & MMLU-Pro & 71.3 & 72.5 & 71.6 & \textbf{73.8} \\
\cmidrule(lr){2-7}
& \multirow{2}{*}{Qwen3-4B-Instruct} & SciKnowEval & 36.0 & 38.0 & \textbf{44.0} & 32.0 \\
& & MMLU-Pro & 84.5 & 84.1 & \textbf{84.9} & 84.8 \\
\midrule
\multirow{4}{*}{Medical}
& \multirow{2}{*}{Qwen2.5-7B-Instruct} & MedQA & 61.5 & 70.3 & 72.1 & \textbf{74.8} \\
& & MedMCQA & 56.5 & 64.5 & 66.6 & \textbf{69.8} \\
\cmidrule(lr){2-7}
& \multirow{2}{*}{Qwen3-4B-Instruct} & MedQA & 72.7 & 71.1 & 70.6 & \textbf{76.8} \\
& & MedMCQA & 60.2 & 59.8 & 60.1 & \textbf{61.6} \\
\midrule
\multirow{4}{*}{Coding}
& \multirow{2}{*}{Qwen2.5-7B-Instruct} & HE+ & 76.0 & 76.2 & 76.8 & \textbf{77.8} \\
& & MultiPL-E & 38.3 & \textbf{39.9} & 38.2 & 38.5 \\
\cmidrule(lr){2-7}
& \multirow{2}{*}{Qwen3-4B-Instruct} & HE+ & 83.1 & 84.2 & 81.8 & \textbf{85.1} \\
& & MultiPL-E & 55.9 & 55.2 & 53.3 & \textbf{61.1} \\
\bottomrule
\end{tabular}
\caption{Self-distillation results across biology, medical question answering, and coding. AOPD outperforms GRPO and OPSD across all three domains.}
\label{table:self_instmodels_working}
\end{table*}


\begin{figure*}[t]
    \centering
    \begin{subfigure}{0.31\linewidth}
        \includegraphics[width=\linewidth]{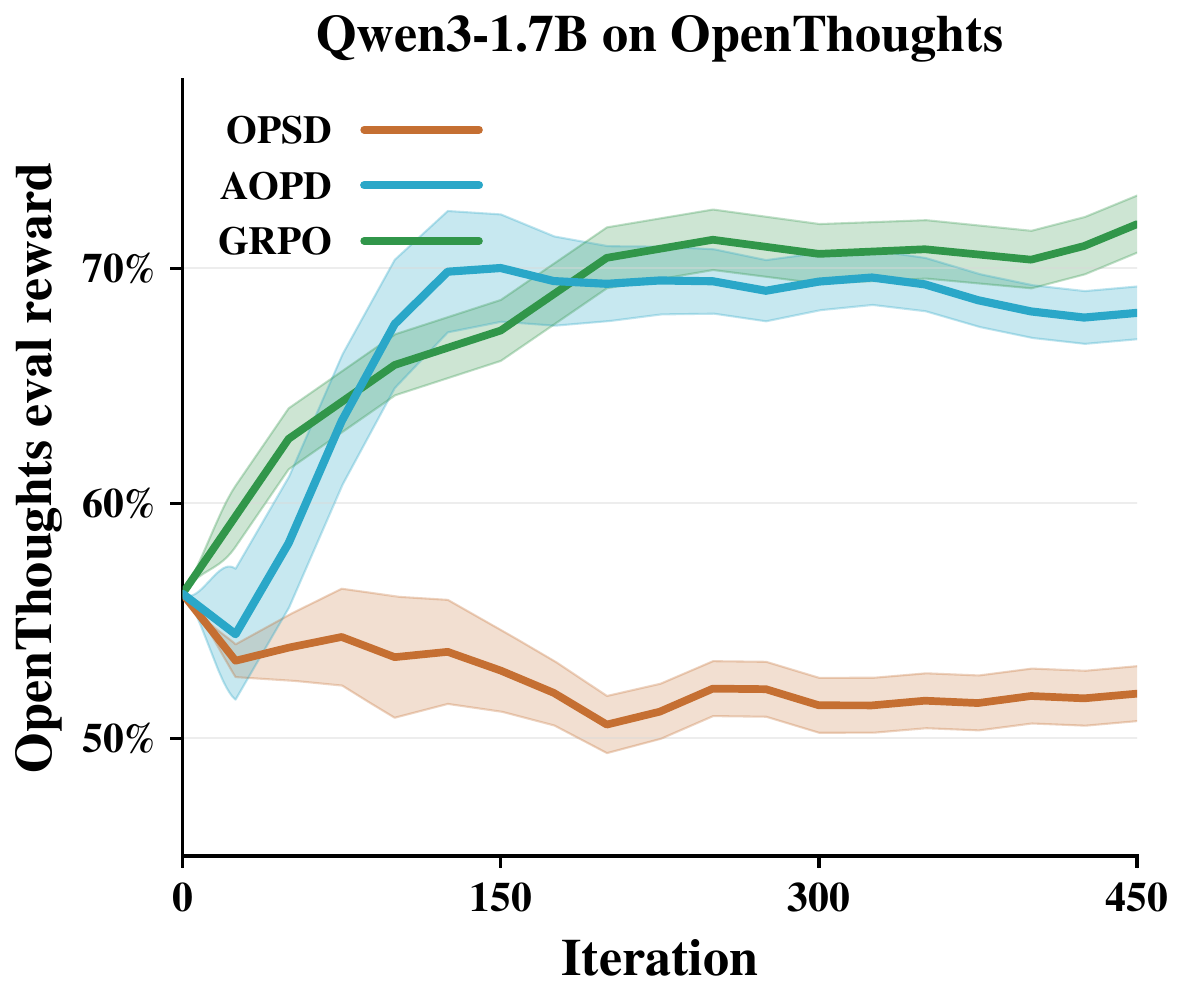}
    \end{subfigure}%
    \begin{subfigure}{0.31\linewidth}
        \includegraphics[width=\linewidth]{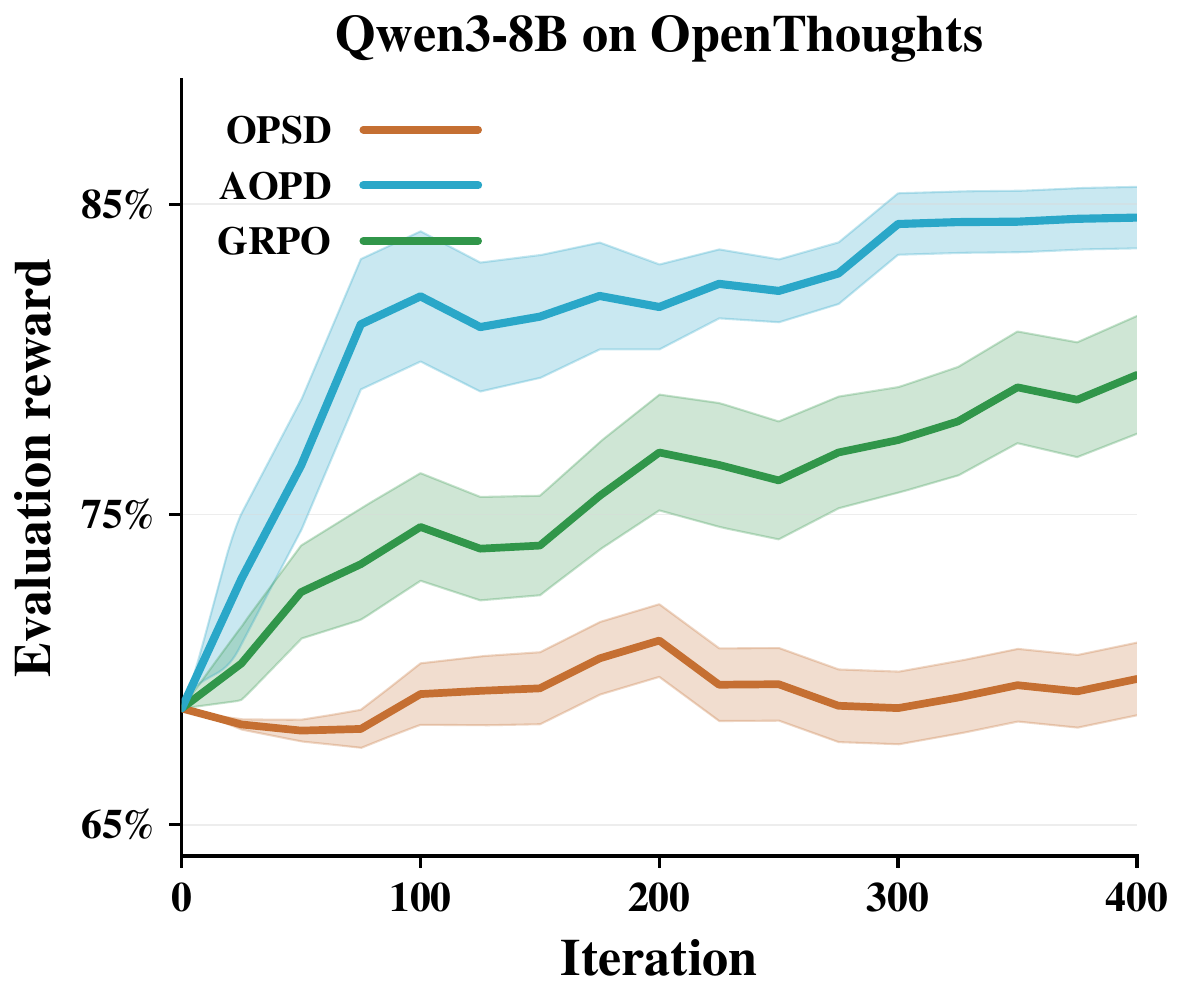}
    \end{subfigure}%
    \begin{subfigure}{0.31\linewidth}
        \includegraphics[width=\linewidth]{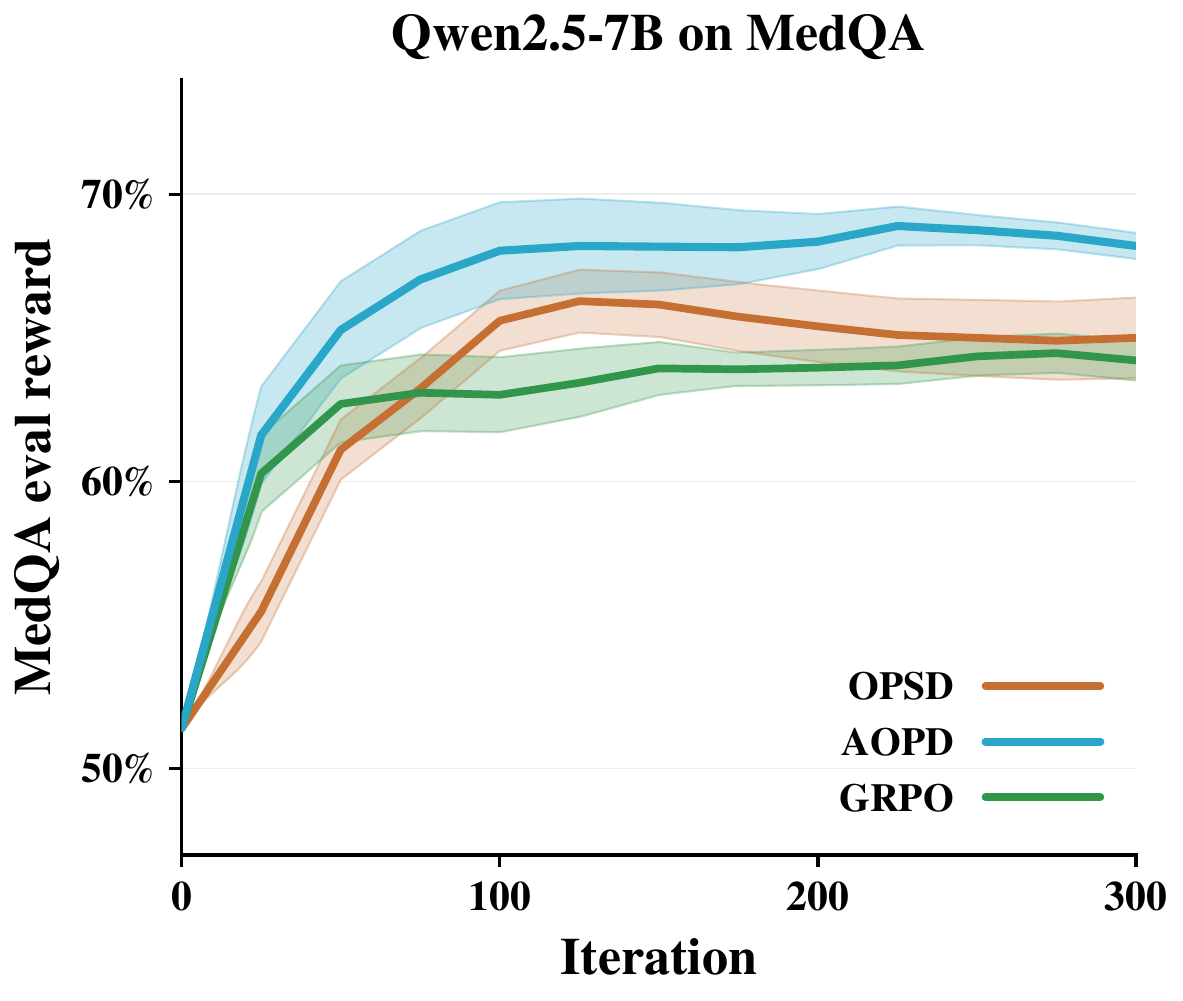}
    \end{subfigure}
    \vspace{-7pt}
    \caption{Training progress for GRPO, OPSD, and AOPD in self-distillation. \textbf{(Left/middle)} Open-Thoughts evaluation reward for Qwen3-1.7B and Qwen3-8B. \textbf{(Right)} MedQA test accuracy after training on MedMCQA with Qwen2.5-7B-Instruct.}
    \label{fig:OPSD_results}
\end{figure*}

\subsection{Ablation: Is the teacher always useful?}
To test whether teacher-preferred actions necessarily benefit the student, we conduct a single-token intervention study with Qwen3-1.7B as the student and Qwen3-4B-Instruct-2507 as the teacher. We also consider a self-distillation setting, where the teacher is Qwen3-1.7B with access to privileged information. For a set of 512 student response on Open-Thoughts, we select a position 
either \textbf{(i)} uniformly at random or \textbf{(ii)} by maximizing the KL divergence \(\KL(T(\cdot|h) \|S(\cdot|h))\) between the teacher and student distributions. We replace the student’s token with the teacher’s highest-probability token, then let the student complete the response greedily under the same 8192-token total budget. We compare the final answer of this intervention with the original response. 

The results are shown in \Cref{tab:teacher_intervention}. Teacher interventions cause both improvements and failures: with choosing the highest-disagreement position, teacher actions help (i.e., changing the reward to 1) in 9.6\% of cases but harm (i.e., changing the reward to 0) in 9.8\%. The high unchanged rate is expected as single-token interventions on the entire trajectory can rarely flip the final outcomes. 

Interestingly, the privileged self-teacher provides more helpful interventions than the  external teacher at uniformly random positions (8.6\% vs. 6.9\%).  This suggests that self-distillation can provide more beneficial guidance in some settings, despite using a smaller model.  This further highlights the importance of selecting teacher guidance based on its contribution to the student's reward rather than the teacher's preferences alone.

\begin{table}[t]
\centering
\begin{tabular}{llccc}
\toprule
Teacher & Intervention position & Helpful (\%) & Unchanged (\%) & Harmful (\%) \\
\midrule
\multirow{2}{*}{External}
    & Random                  & 6.9 & 85.1 & 8.0 \\
    & Highest KL disagreement & 9.6 & 80.6 & 9.8 \\
\midrule
\multirow{2}{*}{Self-distilled}
    & Random                  & 8.6 & 83.4 & 8.0 \\
    & Highest KL disagreement & 9.0 & 81.4 & 9.6 \\
\bottomrule
\end{tabular}
\caption{Effect of replacing a student token with the teacher's preferred token under external-teacher and self-distillation settings. Interventions can both improve and harm the student's final outcome.}
\label{tab:teacher_intervention}
\end{table}



\vspace{-4pt}
\section{Theory: Reward Improvement with Advantage-based Distillation}
\label{sec:atd-theory}

Although a teacher policy may achieve a higher reward than the student when evaluated as a complete policy, this does not imply that every partial update toward the teacher improves the student. In practice, training does not replace the student policy with the teacher policy directly; rather, it proceeds through a finite sequence of updates, often with a student whose capacity is limited relative to that of the teacher. Consequently, the teacher’s eventual advantage does not establish that intermediate training updates improve task performance.

To characterize the effect of training on the student, we analyze a single update under different objectives and examine whether it increases the student’s expected final reward.
\vspace{-4pt}
\paragraph{A proximal view of distillation.}
Following \Cref{sec:method}, let \(S\) denote the current, frozen student model, \(T\) the teacher model, and \(A^S\) the student advantage function. Let \(R(x,y)\in[0,1]\) denote the terminal reward obtained by response \(y\) for prompt \(x\). We define the expected task reward of a policy $\pi$ as
\(
\calR(\pi)
=
\E_{x,\,y\sim\pi(\cdot\mid x)}[R(x,y)]
\).

Our goal is to analyze a limited update from the current frozen student $S$ to to the updated policy $\pi$, rather than assuming the policy can move far arbitrary. To model this limited update from the current student, we follow standard trust-region analyses \citep{schulman2015trust,tomar2020mirror} and the proximal analysis in \cite{sabbaghi2026infosft}. We restrict the KL movement from $S$ with the following constraint:
\[
\KL(S\|\pi) =   \frac{1}{H}\sum_{t=0}^{H-1}
\E_{h_t\sim S(\cdot) }
\left[
\KL\bigl(S(\cdot\mid h_t)\,\|\,\pi(\cdot\mid h_t)\bigr)
\right]
\le \epsilon
\]
Adding this constraint to the optimization for objective $J$ provides a controlled notion of how much the student can change in one update. Writing the Lagrangian of this constrained problem gives the proximal update:
\begin{equation}
\label{eq:theory-prox}
S^+
\in
\arg\max_{\pi}
\left\{
J(\pi)-\frac{1}{\eta}\KL(S\|\pi)
\right\},
\end{equation}
where $J$ is either the FKL-OPD objective in \Cref{eq:fkl_max} or the AOPD objective in \Cref{eq:step_AOPD}. Here $1/\eta$ is the dual variable for $\epsilon$, so smaller $\eta$ gives a more conservative update. We conduct the analysis in policy space, characterizing the local update induced by each objective rather than the parameter-space dynamics of SGD. Accordingly, we write the student policy as \(\pi\) in place of \(S_\theta\). Throughout, we hold the teacher, rollout distribution, and student advantage function \(A^S\) fixed at their values for the current student policy. Under this proximal-update framework, the following result compares FKL-OPD with AOPD:

\begin{theorem}[Reward Improvement from a Single Proximal Update]
\label{thm:atd-improvement}
\textbf{(i)} For sufficiently small $\eta>0$, the exact proximal updates satisfy the following relation at every reachable history $h$:
\vspace{-2pt}
\begin{equation}\label{eq:change_FKL}
    \E_{a\sim S_{\mathrm{FKL}}^+(\cdot\mid h)}
[A^S(h,a)]
=
\frac{\eta}{1+\eta}
\E_{a\sim T(\cdot\mid h)}
[A^S(h,a)],
\end{equation}
and its sign under OPD remains indeterminate  (see \Cref{cor:atd-hybrid} for concrete example). Whereas,
\vspace{-2pt}
\begin{equation}\label{eq:change_AOPD}
    \E_{a\sim S_{\mathrm{AOPD}}^+(\cdot\mid h)}
    [A^S(h,a)]
    =
    \frac{
    \eta\,\E_{a\sim T(\cdot\mid h)}[A^S(h,a)^2]
    }{
    1+\eta\,\E_{a\sim T(\cdot\mid h)}[A^S(h,a)]
    }
    \ge 0.
\end{equation}
\textbf{(ii)} We have the identity:
\(
\calR(\pi)-\calR(S)=
\E_{x,y\sim\pi}
\left[
\sum_{t=0}^{H-1}A^S(h_t,a_t)
\right]
\).
Consequently 
\(
\calR(S_{\mathrm{AOPD}}^+)\ge\calR(S)
\). The inequality is strict when a teacher-supported action has nonzero advantage at a history reached with positive probability under $S_{\mathrm{AOPD}}^+$.
\end{theorem}

The distinction follows directly from the two update expressions. FKL shifts probability toward the teacher’s actions irrespective of whether those actions improve the current student. Its resulting reward change therefore depends on the teacher’s average advantage under the student, which may be negative even when the teacher achieves a higher overall reward. In contrast, AOPD weights teacher actions by their student advantages, yielding a squared-advantage term and thereby ensuring a nonnegative local reward change.

The small-step condition guarantees that the AOPD proximal update remains a valid probability distribution. In particular, the update cannot remove more probability mass from a negative-advantage action than the current policy assigns to that action. Appendix~\ref{app:proofs} provides the exact proximal solutions and the associated condition on \(\eta\).
\begin{corollary}[Surpassing the teacher]
\label{cor:atd-teacher-crossing}
Suppose $S=T$. For a sufficiently small $\eta>0$, we have
\(
\calR(S_{\mathrm{AOPD}}^+)\ge \calR(T)
\). Therefore, AOPD can still improve the policy and surpass the teacher. Hence, the teacher acts as a source of action-level supervision rather than a fixed target. 
\end{corollary}

\vspace{-6pt}
\paragraph{Reward improvement with clipping.}
The proof of \Cref{thm:atd-improvement} does not require the teacher weights to form a normalized distribution. It only requires them to be nonnegative and that the proximal coefficients  remains positive. The same argument therefore applies to the truncated weights
\(
\widetilde T(a\mid h)=\min\{T(a\mid h),cS(a\mid h)\}.
\)

\begin{corollary}[Reward improvement with clipping]
\label{cor:atd-clipping}
Under the assumptions of \Cref{thm:atd-improvement}, let $S_c^+$ be the exact proximal update for the expected clipped objective in \Cref{eq:AOPD_clipped}. Let $c\ge1$ and $0<\eta<1/c$, we have:
\(
\calR(S_c^+)\ge\calR(S).
\)
\end{corollary}

\vspace{-6pt}
\paragraph{Adding RL does not resolve the imitation conflict.}
An alternative way to incorporate terminal rewards into distillation is to augment the FKL distillation objective with an explicit reward-maximization term:
\begin{equation}
\label{eq:theory-hybrid-objective}
J_{\mathrm{mix}}(\pi)
=
\frac{1}{H} \calR(\pi)+\lambda J_{\mathrm{FKL}}(\pi),
\qquad \lambda>0
\end{equation}
Where $H$ is the horizon and is added to simply puts reward and FKL on the same per-token scale. This objective is studied in previous work such KDRL \citep{xu2025kdrl,zhang2026reinforcement}. The reward term favors successful behavior, but the FKL term still moves the student toward the teacher independently of the student advantage. The two objectives can therefore conflict as we show below:

\begin{lemma}[RL plus FKL-OPD need not improve reward]
\label{cor:atd-hybrid}
For every fixed $\lambda>0$, there exists a two-step task with binary terminal rewards and full-support policies $S,T$ satisfying $\calR(T)>\calR(S)$ such that, for all sufficiently small $\eta>0$,
\[
\calR(S_{\mathrm{mix}}^+)<\calR(S), \quad \calR(S_{\mathrm{OPD}}^+)<\calR(S)
\]
\end{lemma}
\noindent Therefore, adding an RL objective does not necessarily mitigate the harmful aspect of imitation.

\section{Conclusion}
We introduced AOPD, an on-policy distillation objective that asks not only what the teacher prefers, but whether that preference actually helps the current student succeed. By weighting teacher actions with their advantage under the student, AOPD reinforces useful guidance and suppresses harmful imitation, while requiring no additional counterfactual rollouts or auxiliary value models. Across external-teacher and self-distillation settings, AOPD consistently improves over typical distillation and RL baselines, showing that effective distillation should optimize for student success rather than teacher agreement.

\section*{Acknowledgment}
This research was supported by Coefficient Giving, EnCORE (the Institute for Emerging CORE Methods in Data Science), and NSF Award No. 2502489 through the Mathematical Foundations of Alignment in Generative Artificial Intelligence (MFAI) program. AJ was supported in part by  the Sloan fellowship in mathematics, the NSF Award DMS-2311024, an Amazon
Faculty Research Award, an Adobe Faculty Research Award and an iORB grant form USC Marshall School of Business.

\newpage
\bibliography{bibliography}
\bibliographystyle{plainnat}

\newpage
\appendix

\section{Additional Experiments}\label{app:additional_exp}
\paragraph{Group size effect on OPD and AOPD.} For our main experiments in \Cref{sec:experiments}, we choose group-size $G = 4$ for AOPD and OPD. This is because as \Cref{fig:group} shows, increasing the group size leads to only a negligible reward improvement. We also found that this reward improvement does not transfer to better benchmark number. The same effect is also studied for OPD in previous work \cite{hou2026uni}. 

\begin{figure*}[h!]
    \centering
    \begin{subfigure}{0.4\linewidth}
        \includegraphics[width=\linewidth]{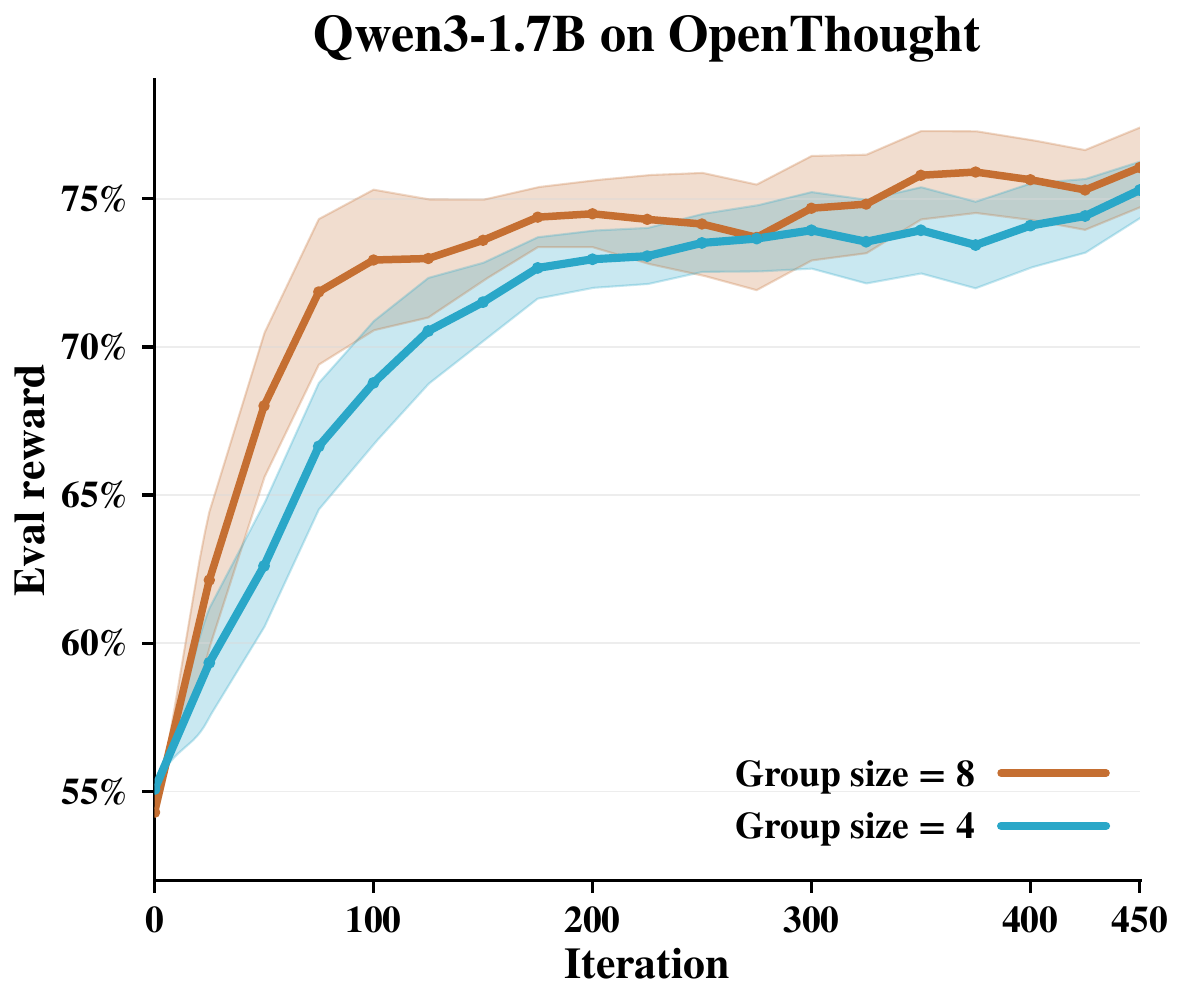}
    \end{subfigure}%
    \hspace{0.4in}
    \begin{subfigure}{0.4\linewidth}
        \includegraphics[width=\linewidth]{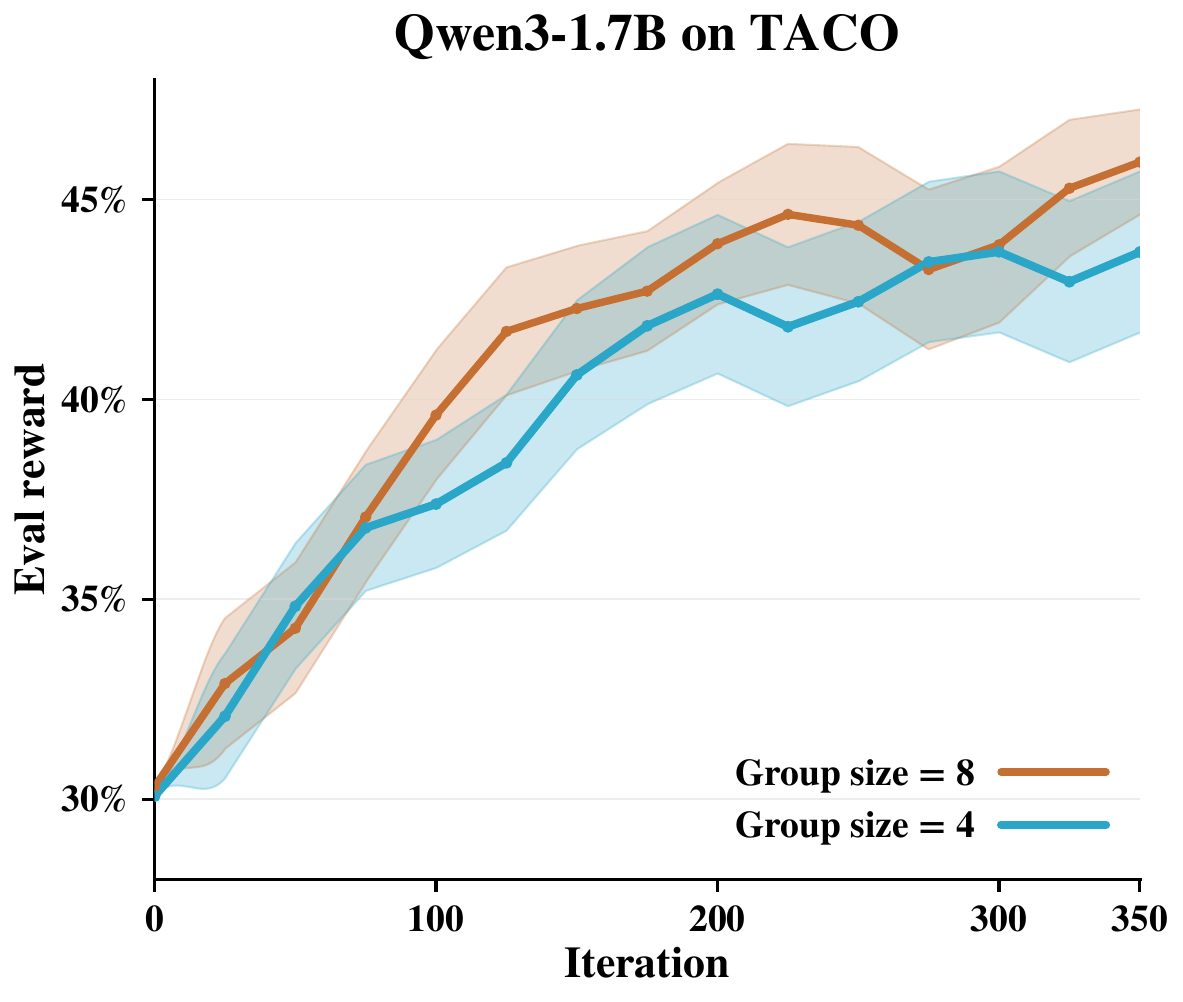}
    \end{subfigure}%
    \vspace{-10pt}
    \caption{Increasing the group size does not significantly improve the performance for AOPD. This is while for methods such GRPO, a larger group size is critical.}
    \label{fig:group}
\end{figure*}

\paragraph{Clipping avoids performance fluctuations.}
As explained in \Cref{sec:method}, we clip the importance sampling weights to control the high variance caused by large importance sampling weight, a standard technique in importance-sampling-based optimization \citep{munos2016safe,wang2016sample,han2019dimension}. This effect is visible in \Cref{fig:clipping}: without clipping, training exhibits substantially larger fluctuations, whereas clipping leads to a more stable optimization trajectory.

\begin{figure*}[h!]
    \centering
    \includegraphics[width=0.45\linewidth]{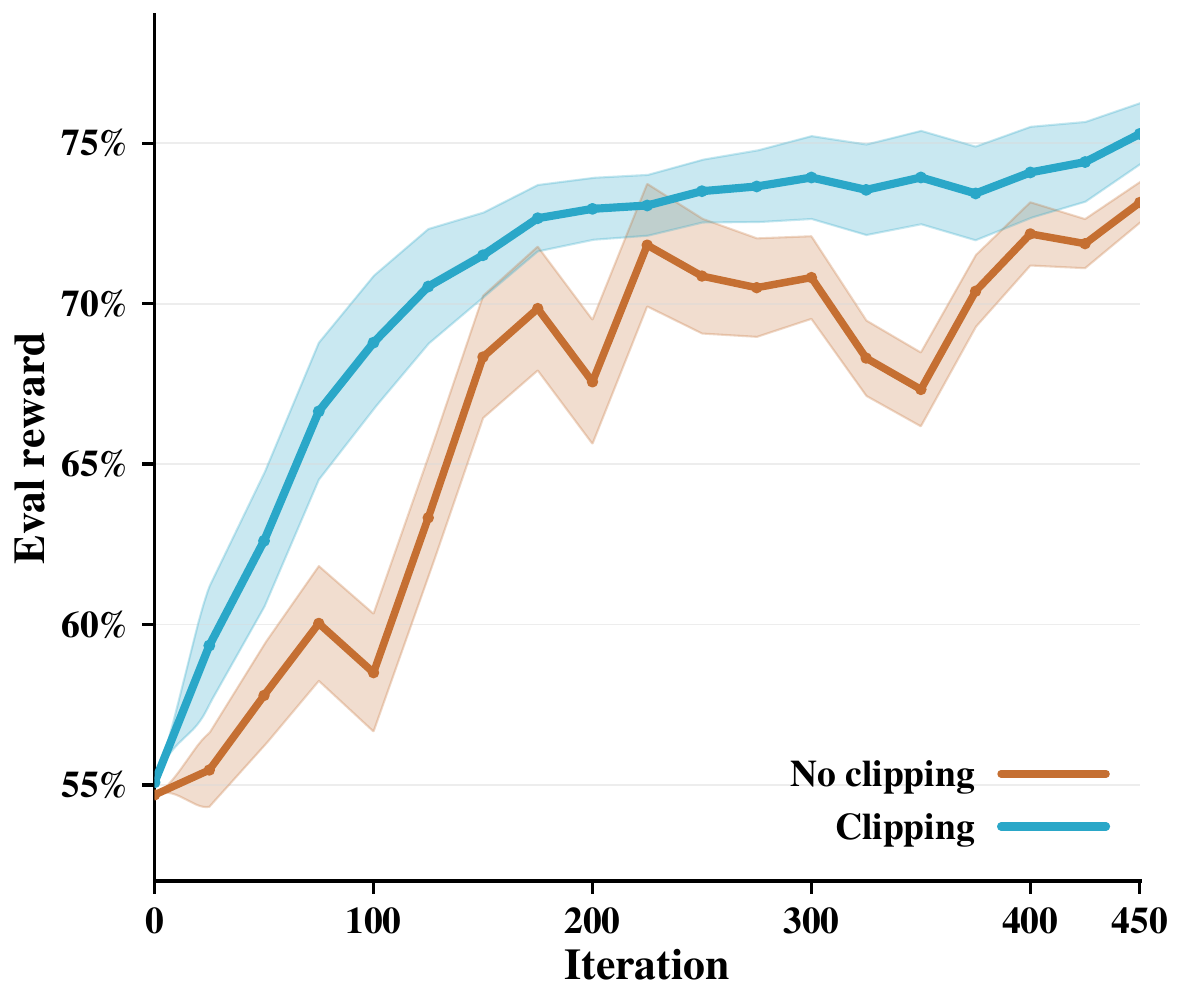}
    \caption{Clipping the importance sampling weight improves the stability and the performance of AOPD.}
    \vspace{-5pt}
    \label{fig:clipping}
\end{figure*}

\section{Proofs: Reward Improvement}
\label{app:proofs}

\begin{proof}[Proof of \Cref{lem:AOPD_unbiased}]
Fix a visited history $h_t=h$. The sampled response comes from $S$, and $b$ is independent of that response conditional on the prompt. Remember that:
\[
\E[R^S\mid h,a_t=a]=Q^S(h,a),
\qquad
\E[R^S\mid h]=V^S(h).
\]
Therefore we have: 
\begin{align*}
\E [
(R^S-b)w_t\log S_\theta(a_t\mid h)
\mid h
] & = \E\left[ \E[
(R^S-b)w_t\log S_\theta(a_t\mid h)
\mid h, a]
\right] \\
& = \E\left[
(Q^S(h, a) -b)w_t\log S_\theta(a_t\mid h)
\right] \\
& =
\sum_{a\in\calV}
S(a\mid h)\bigl(Q^S(h,a)-b\bigr)
\frac{T(a\mid h)}{S(a\mid h)}
\log S_\theta(a\mid h)\\
& =
\sum_{a\in\calV}
T(a\mid h)\bigl(Q^S(h,a)-b\bigr)
\log S_\theta(a\mid h)
\end{align*}
where the support assumption ($T \ll S$) lets us apply the Radon-Nikodym theorem in the last line. Meanwhile, $\E_{a\sim T} [\log S_\theta(a | h_t)]$ depends only on the history and in independent of $a$, so:
\begin{align*}
\E\left[(R^S-b) \E_{a\sim T} [\log S_\theta(a \mid h_t)] \mid h \right]
= &
\bigl(V^S(h)-b\bigr) \E_{a\sim T} [\log S_\theta(a \mid h_t)]\\
= &  \E_{a\sim T} \big[ \bigl(V^S(h)-b\bigr) \log S_\theta(a \mid h_t) \big]
\end{align*}
Subtracting the two expressions cancels the baseline:
\begin{align*}
&\E\left[
(R^S-b)\Big(w_t\log S_\theta(a_t\mid h)-\E_{a\sim T} [\log S_\theta(a \mid h_t)]\Big)
\mid h,b
\right]\\
&\quad=
\sum_{a\in\calV}
T(a\mid h)\bigl(Q^S(h,a)-V^S(h)\bigr)
\log S_\theta(a\mid h)
\end{align*}
Averaging over the baseline, histories, and token positions yields $J_{\mathrm{AOPD}}(\theta)$. Differentiating this identity under the stated conditions gives the gradient result since the expectation and the gradient operations can be swapped.
\end{proof}

\begin{proof}[Proof of \Cref{thm:atd-improvement}]
We first prove \Cref{eq:change_FKL,eq:change_AOPD}. Then we show a general identity is true: $\E_{a \sim \pi}[A^S(h, a)] = \calR(\pi) - \calR(S)$, and use it to prove the final statement.

\textbf{(i)}

\paragraph{Forward-KL.}
Recall that FKL optimizes the objective in \Cref{eq:fkl_max}. Replacing the updated student $S_\theta$ by the candidate policy $\pi$, the proximal objective in \Cref{eq:theory-prox} becomes
\begin{align*}
J_{\mathrm{FKL}}(\pi)-\frac{1}{\eta}\KL(S\|\pi)
&=
\E_x\left[
\frac{1}{H}\sum_{t=0}^{H-1}
\E_{h_t\sim S(\cdot)}
\left[
\sum_{a\in\calV}
T(a\mid h_t)\log\pi(a\mid h_t)
\right.\right.\\
&\qquad\qquad\left.\left.
-\frac{1}{\eta}
\sum_{a\in\calV}
S(a\mid h_t)
\log\frac{S(a\mid h_t)}{\pi(a\mid h_t)}
\right]
\right]\\
&=
\E_x\left[
\frac{1}{H}\sum_{t=0}^{H-1}
\E_{h_t\sim S(\cdot)}
\left[
\sum_{a\in\calV}
\left(
T(a\mid h_t)+\frac{1}{\eta}S(a\mid h_t)
\right)
\log\pi(a\mid h_t)
\right.\right.\\
&\qquad\qquad\left.\left.
-\frac{1}{\eta}
\sum_{a\in\calV}
S(a\mid h_t)\log S(a\mid h_t)
\right]
\right].
\end{align*}
We can drop the last term since it is independent of $\pi$. Then, this can be optimized for each $h$ separately. Therefore, fixing a history $h$, the relevant objective is:
\begin{equation}\label{eq:FKL_simplified}
    \frac{1}{\eta}
\sum_{a\in\calV}
\left[
S(a\mid h)+\eta T(a\mid h)
\right]
\log\pi(a\mid h)
\end{equation}
The optimization separates across histories. Since:
\[
\sum_{a\in\calV}
\left[S(a\mid h)+\eta T(a\mid h)\right]
=
1+\eta
\]
maximizing over $\pi(\cdot\mid h)$ subject to $\sum_a\pi(a\mid h)=1$ gives:
\[
S_{\mathrm{FKL}}^+(a\mid h)
=
\frac{S(a\mid h)+\eta T(a\mid h)}{1+\eta}.
\]

By definition of the student advantage:
\[
\E_{a\sim S(\cdot\mid h)}[A^S(h,a)]
=
\E_{a\sim S}[Q^S(h,a)]-V^S(h)
=
0.
\]
Therefore:
\begin{align*}
\E_{a\sim S_{\mathrm{FKL}}^+(\cdot\mid h)}
[A^S(h,a)]
&=
\frac{
\E_{a\sim S}[A^S(h,a)]
+
\eta\E_{a\sim T}[A^S(h,a)]
}{1+\eta}\\
&=
\frac{\eta}{1+\eta}
\E_{a\sim T}[A^S(h,a)].
\end{align*}
Thus, the reward effect of the FKL update depends on the teacher's average advantage under the current student, which can have either sign.

\paragraph{AOPD.}
For AOPD, substituting \Cref{eq:step_AOPD} into the same proximal objective gives
\begin{align*}
J_{\mathrm{AOPD}}(\pi)-\frac{1}{\eta}\KL(S\|\pi)
&=
\E_x\left[
\frac{1}{H}\sum_{t=0}^{H-1}
\E_{h_t\sim S(\cdot)}
\left[
\sum_{a\in\calV}
T(a\mid h_t)A^S(h_t,a)\log\pi(a\mid h_t)
\right.\right.\\
&\qquad\qquad\left.\left.
-\frac{1}{\eta}
\sum_{a\in\calV}
S(a\mid h_t)
\log\frac{S(a\mid h_t)}{\pi(a\mid h_t)}
\right]
\right].
\end{align*}
Similar to \Cref{eq:FKL_simplified}, dropping terms independent of $\pi$, the objective at a fixed history $h$ becomes:
\[
\frac{1}{\eta}
\sum_{a\in\calV}
\left[
S(a\mid h)
+
\eta T(a\mid h)A^S(h,a)
\right]
\log\pi(a\mid h).
\]
For sufficiently small $\eta$, all coefficients are positive. Their sum is
\begin{align*}
\sum_{a\in\calV}
\left[
S(a\mid h)+\eta T(a\mid h)A^S(h,a)
\right]
&=
1+\eta\sum_{a\in\calV}T(a\mid h)A^S(h,a)\\
&=
1+\eta\E_{a\sim T}[A^S(h,a)].
\end{align*}
Hence the proximal optimizer is
\begin{align}\label{eq:proximal-S}
S_{\mathrm{AOPD}}^+(a\mid h)
=
\frac{
S(a\mid h)+\eta T(a\mid h)A^S(h,a)
}{
1+\eta\E_{a\sim T}[A^S(h,a)]
}.
\end{align}
Its expected student advantage is therefore
\begin{align*}
\E_{a\sim S_{\mathrm{AOPD}}^+(\cdot\mid h)}
[A^S(h,a)]
&=
\frac{
\E_{a\sim S}[A^S(h,a)]
+
\eta\E_{a\sim T}[A^S(h,a)^2]
}{
1+\eta\E_{a\sim T}[A^S(h,a)]
}\\
&=
\frac{
\eta\E_{a\sim T}[A^S(h,a)^2]
}{
1+\eta\E_{a\sim T}[A^S(h,a)]
}
\ge 0.
\end{align*}
The denominator is positive by the small-step condition, while the numerator is nonnegative because it is an expectation of squared advantages.

\paragraph{(ii) From local advantage to final reward.}
We finally show that the nonnegative advantage at each history implies a nonnegative change in the final task reward. Recall that $A^S$ is defined using continuations under the original student $S$, but it can be evaluated at any history. This is because we assumed $S$ has full support and the sufficiently small proximal update preserves this support.
For any history-action pair:
\[
Q^S(h_t,a_t)=V^S(h_{t+1}),
\]
where $V^S(h_H)=R(x,y)$ at the terminal history. Thus, along any complete response, regardless of the policy that generated it,
\begin{align*}
\sum_{t=0}^{H-1}A^S(h_t,a_t)
&=
\sum_{t=0}^{H-1}
\left(Q^S(h_t,a_t)-V^S(h_t)\right)\\
&=
\sum_{t=0}^{H-1}
\left(V^S(h_{t+1})-V^S(h_t)\right)\\
&=
R(x,y)-V^S(x).
\end{align*}
Taking the expectation over responses generated by any policy $\pi$, and using
\[
\E_x[V^S(x)]=\mathcal R(S),
\]
gives
\begin{align*}
\mathcal R(\pi)-\mathcal R(S)
&=
\E_{x,y\sim\pi}
\left[
\sum_{t=0}^{H-1}A^S(h_t,a_t)
\right]\\
&=
\sum_{t=0}^{H-1}
\E_{h_t\sim\pi(\cdot)}
\left[
\E_{a\sim\pi(\cdot\mid h_t)}
[A^S(h_t,a)]
\right].
\end{align*}
Setting $\pi=S_{\mathrm{AOPD}}^+$, every inner expectation is nonnegative by the local result above. Therefore,
\[
\mathcal R(S_{\mathrm{AOPD}}^+)\ge\mathcal R(S).
\]
The inequality is strict whenever $S_{\mathrm{AOPD}}^+$ reaches with positive probability a history $h$ for which $T(a\mid h)>0$ and $A^S(h,a)\neq0$ for some action $a$.

\end{proof}

\begin{proof}[Proof of \Cref{cor:atd-clipping}]
Replacing $T$ by $\widetilde T$ in the AOPD proximal objective (see~\eqref{eq:proximal-S}) gives
\[
S_c^+(a\mid h)
=
\frac{
S(a\mid h)+\eta\widetilde T(a\mid h)A^S(h,a)
}{
1+\eta\sum_b\widetilde T(b\mid h)A^S(h,b)
}.
\]
Since $R\in[0,1]$, we have $A^S(h,a)\in[-1,1]$. Moreover,
\[
\widetilde T(a\mid h)\le cS(a\mid h),
\]
so for $\eta<1/c$,
\[
S(a\mid h)+\eta\widetilde T(a\mid h)A^S(h,a)
\ge
S(a\mid h)(1-\eta c)>0.
\]
Therefore the update is well defined. Using $\E_{a\sim S}[A^S(h,a)]=0$,
\[
\E_{a\sim S_c^+}[A^S(h,a)]
=
\frac{
\eta\sum_a\widetilde T(a\mid h)A^S(h,a)^2
}{
1+\eta\sum_a\widetilde T(a\mid h)A^S(h,a)
}
\ge0.
\]
Applying the same performance-difference argument as in \Cref{thm:atd-improvement} gives
\[
\mathcal R(S_c^+)\ge\mathcal R(S).
\]
\end{proof}

\subsection{Surpassing the teacher}
\begin{proof}[Proof of \Cref{cor:atd-teacher-crossing}]
Setting $S=T$ in \Cref{thm:atd-improvement} gives
\[
\mathcal R(S_{\mathrm{AOPD}}^+)\ge\mathcal R(S)=\mathcal R(T).
\]
The strict-improvement condition follows directly from the theorem. Meanwhile, the FKL update becomes
\[
S_{\mathrm{FKL}}^+
=
\frac{T+\eta T}{1+\eta}
=
T.
\]
\end{proof}

This follows directly from \Cref{thm:atd-improvement} by setting $S=T$. In contrast, FKL makes no update once the student matches the teacher. AOPD can therefore continue improving beyond the teacher by reweighting its actions according to their task advantage.

\subsection{Adding RL does not guarantee improvement}
\begin{proof}[Proof of \Cref{cor:atd-hybrid}]
In order to prove that FKL + RL (and FKL) can reduce the reward after the update, we first expand the updated policy $pi$ for infinitesimal $\eta$. Under this setting, we write down the change in the student reward. We then construct a counter-example that the update reward is smaller, even though the teacher is stronger itself and has a high reward.

\paragraph{The hybrid update.}
Recall that
\[
J_{\mathrm{mix}}(\pi)
=
\frac{1}{H}\mathcal R(\pi)
+
\lambda J_{\mathrm{FKL}}(\pi).
\]
Its proximal update therefore maximizes
\[
\frac{1}{H}\mathcal R(\pi)
+
\lambda J_{\mathrm{FKL}}(\pi)
-
\frac{1}{\eta}\KL(S\|\pi).
\]
For sufficiently small $\eta$, the KL penalty keeps the updated policy close to $S$, so we write
\[
\pi(a\mid h)
=
S(a\mid h)+\eta v(a\mid h)+O(\eta^2),
\qquad
\sum_a v(a\mid h)=0.
\]

\noindent We now expand the three terms around $S$. From the performance-difference identity and $\E_{a\sim S}[A^S(h,a)]=0$:
\[
\frac{1}{H}\bigl(\mathcal R(\pi)-\mathcal R(S)\bigr)
=
\frac{\eta}{H}
\sum_t
\E_{h_t\sim S(\cdot)}
\left[
\sum_a v(a\mid h_t)A^S(h_t,a)
\right]
+O(\eta^2)
\]
Similarly, using $\log(1+z)=z+O(z^2)$,
\begin{align*}
J_{\mathrm{FKL}}(\pi)-J_{\mathrm{FKL}}(S)
&=
\frac{\eta}{H}
\sum_t
\E_{h_t\sim S(\cdot)}
\left[
\sum_a
\frac{T(a\mid h_t)}{S(a\mid h_t)}
v(a\mid h_t)
\right]
+O(\eta^2),\\
\KL(S\|\pi)
&=
\frac{\eta^2}{2H}
\sum_t
\E_{h_t\sim S(\cdot)}
\left[
\sum_a
\frac{v(a\mid h_t)^2}{S(a\mid h_t)}
\right]
+O(\eta^3),
\end{align*}
where the first-order KL term vanishes because $\sum_a v(a\mid h)=0$.
Therefore, at each history $h$, the leading-order direction maximizes:
\[
\sum_a v(a\mid h)A^S(h,a)
+
\lambda\sum_a\frac{T(a\mid h)}{S(a\mid h)}v(a\mid h)
-
\frac12\sum_a\frac{v(a\mid h)^2}{S(a\mid h)}
\]
subject to $\sum_a v(a\mid h)=0$. Introducing a multiplier $\nu$, the first-order condition gives:
\[
A^S(h,a)
+
\lambda\frac{T(a\mid h)}{S(a\mid h)}
-
\frac{v(a\mid h)}{S(a\mid h)}
-\nu
=0.
\]
Summing over $a$ and using $\E_{a\sim S}[A^S(h,a)]=0$ and $\sum_aS(a\mid h)=\sum_aT(a\mid h)=1$ gives $\nu=\lambda$. Hence
\begin{equation}
\label{eq:proof-hybrid-direction}
v(a\mid h)
=
S(a\mid h)A^S(h,a)
+
\lambda\bigl(T(a\mid h)-S(a\mid h)\bigr).
\end{equation}
The first term is the reward-improving direction, while the second moves the student toward the teacher.

\paragraph{Effect on the student's reward.}
Substituting \Cref{eq:proof-hybrid-direction} into the first-order reward expansion gives
\begin{align*}
\mathcal R(S_{\mathrm{mix}}^+)-\mathcal R(S)
&=
\eta
\sum_t
\E_{h_t\sim S(\cdot)}
\left[
\sum_a S(a\mid h_t)A^S(h_t,a)^2
+
\lambda\sum_a\bigl(T(a\mid h_t)-S(a\mid h_t)\bigr)A^S(h_t,a)
\right]
+O(\eta^2)\\
&=
\eta
\sum_t
\E_{h_t\sim S(\cdot)}
\left[
\E_{a\sim S}[A^S(h_t,a)^2]
+
\lambda\E_{a\sim T}[A^S(h_t,a)]
\right]
+O(\eta^2),
\end{align*}
where the second equality again uses $\E_{a\sim S}[A^S(h,a)]=0$. Thus,
\begin{equation}
\label{eq:proof-hybrid-reward}
\mathcal R(S_{\mathrm{mix}}^+)-\mathcal R(S)
=
\eta
\sum_t
\E_{h_t\sim S(\cdot)}
\left[
\E_{a\sim S}[A^S(h_t,a)^2]
+
\lambda\E_{a\sim T}[A^S(h_t,a)]
\right]
+O(\eta^2).
\end{equation}
The RL contribution is always nonnegative, but the FKL contribution can be negative.

\paragraph{A two-step counterexample.}
Consider a two-step task where the first action chooses a "safe" or "difficult" route and the second action determines success or failure. Let $0<\varepsilon<1/3$. Define the student by:
\[
S(\mathrm{difficult})=\varepsilon,
\qquad
S(\mathrm{success}\mid\mathrm{difficult})=\varepsilon,
\qquad
S(\mathrm{success}\mid\mathrm{safe})=1-\varepsilon
\]
and the teacher by:
\[
T(\mathrm{difficult})=1-\varepsilon,
\qquad
T(\mathrm{success}\mid\mathrm{difficult})
=
T(\mathrm{success}\mid\mathrm{safe})
=
1-\varepsilon
\]
The expected rewards for the student and the teacher are: 
\[
\mathcal R(S)
=
\varepsilon^2+(1-\varepsilon)^2
=
1-2\varepsilon+2\varepsilon^2,
\qquad
\mathcal R(T)=1-\varepsilon,
\]
so
\[
\mathcal R(T)-\mathcal R(S)
=
\varepsilon(1-2\varepsilon)>0.
\]
Thus, the teacher is strictly better.
At the initial history $h_0$:
\[
Q^S(h_0,\mathrm{difficult})=\varepsilon,
\qquad
Q^S(h_0,\mathrm{safe})=1-\varepsilon,
\]
and therefore
\[
A^S(h_0,\mathrm{difficult})
=
-(1-\varepsilon)(1-2\varepsilon),
\qquad
A^S(h_0,\mathrm{safe})
=
\varepsilon(1-2\varepsilon).
\]
Hence
\[
\E_{a\sim T}[A^S(h_0,a)]
=
-(1-2\varepsilon)^2.
\]

At the difficult-route history,
\[
\E_{a\sim T}[A^S(h_{\mathrm{difficult}},a)]
=
1-2\varepsilon,
\]
and this history is reached by $S$ with probability $\varepsilon$. At the safe-route history, teacher and student have the same success probability, so the expected teacher advantage is zero. Therefore:
\begin{equation}
\label{eq:proof-example-teacher-advantage}
\sum_t
\E_{h_t\sim S(\cdot)}
\E_{a\sim T}[A^S(h_t,a)]
=
-(1-2\varepsilon)^2+\varepsilon(1-2\varepsilon)
=
(1-2\varepsilon)(3\varepsilon-1).
\end{equation}

For the squared-advantage term, the initial history contributes
\[
\E_{a\sim S}[A^S(h_0,a)^2]
=
\varepsilon(1-\varepsilon)(1-2\varepsilon)^2.
\]
The two second-step histories together contribute $\varepsilon(1-\varepsilon)$. Thus,
\begin{equation}
\label{eq:proof-example-squared-advantage}
\sum_t
\E_{h_t\sim S(\cdot)}
\E_{a\sim S}[A^S(h_t,a)^2]
=
\varepsilon(1-\varepsilon)
\left[
1+(1-2\varepsilon)^2
\right].
\end{equation}

Substituting \Cref{eq:proof-example-teacher-advantage,eq:proof-example-squared-advantage} into \Cref{eq:proof-hybrid-reward}, the coefficient of $\eta$ is
\[
\varepsilon(1-\varepsilon)
\left[
1+(1-2\varepsilon)^2
\right]
+
\lambda(1-2\varepsilon)(3\varepsilon-1).
\]
As $\varepsilon\to0$, this converges to $-\lambda<0$. Therefore, for every fixed $\lambda>0$, we can choose a sufficiently small $\varepsilon>0$ such that the linear reward change is negative. Holding this task fixed, the negative $O(\eta)$ term dominates the $O(\eta^2)$ remainder for sufficiently small $\eta>0$, giving
\[
\mathcal R(S_{\mathrm{mix}}^+)<\mathcal R(S).
\]
\end{proof}

\end{document}